\documentclass{article}

\usepackage{PRIMEarxiv}

\usepackage[utf8]{inputenc} % allow utf-8 input
\usepackage[T1]{fontenc}    % use 8-bit T1 fonts
\usepackage{hyperref}       % hyperlinks
\usepackage{url}            % simple URL typesetting
\usepackage{booktabs}       % professional-quality tables
\usepackage{amsfonts}       % blackboard math symbols
\usepackage{nicefrac}       % compact symbols for 1/2, etc.
\usepackage{microtype}      % microtypography
\usepackage{lipsum}
\usepackage{fancyhdr}       % header
\usepackage{graphicx}       % graphics
\usepackage{amsmath}
\usepackage{amsthm}
\usepackage{amssymb}
\usepackage[linewidth=1pt]{mdframed}
\usepackage{dsfont}
\usepackage{algorithm}
\usepackage{algorithmicx}
\usepackage{slashbox}
\usepackage{scalerel}
\newtheorem{example}{Example}%
\newtheorem{remark}{Remark}%
\newtheorem{definition}{Definition}%
\newtheorem{theorem}{Theorem}%  meant for continuous numbers
\newtheorem{lemma}{Lemma}

\newcommand{\vectornorm}[1]{\|#1\|}

\def\calX{{\cal X}}

\newif\ifCOM
\COMfalse

\title{Computing a partition function of a generalized pattern-based energy over a semiring}

\author{
  Rustem Takhanov\\
  Mathematics Department, School of Sciences and Humanities \\
  Nazarbayev University \\
  Astana\\
  \texttt{rustem.takhanov@nu.edu.kz} \\
}

\begin{document}
\maketitle

\begin{abstract}
Valued constraint satisfaction problems with ordered variables (VCSPO) are a special case of Valued CSPs in which variables are totally ordered and soft constraints are imposed on tuples of variables that do not violate the order. %It is known that the computational complexity of such problems, parameterized by a set of allowed constraints, is the same as the complexity of VCSPs with unordered variables.
We study a restriction of VCSPO, in which soft constraints are imposed on a segment of adjacent variables and a constraint language $\Gamma$ consists of $\{0,1\}$-valued characteristic functions of predicates. 
This kind of potentials generalizes the so-called pattern-based potentials, which were applied in many tasks of structured prediction.

For a constraint language $\Gamma$ we introduce a closure operator, $ \overline{\Gamma^{\cap}}\supseteq \Gamma$, and give examples of constraint languages for which $|\overline{\Gamma^{\cap}}|$ is small. If all predicates in $\Gamma$ are cartesian products, we show that the minimization of a generalized pattern-based potential (or, the computation of its partition function) can be made in ${\mathcal O}(|V|\cdot |D|^2 \cdot |\overline{\Gamma^{\cap}}|^2 )$ time, where $V$ is a set of variables, $D$ is a domain set. If, additionally, only non-positive weights of constraints are allowed, the complexity of the minimization task drops to ${\mathcal O}(|V|\cdot  |\overline{\Gamma^{\cap}}| \cdot |D| \cdot \max_{\varrho\in \Gamma}\vectornorm{\varrho}^2 )$ where $\vectornorm{\varrho}$ is the arity of $\varrho\in \Gamma$.
For a general language $\Gamma$ and non-positive weights, the minimization task can be carried out in ${\mathcal O}(|V|\cdot |\overline{\Gamma^{\cap}}|^2)$ time.

We argue that in many natural cases $\overline{\Gamma^{\cap}}$ is of moderate size, though in the worst case $|\overline{\Gamma^{\cap}}|$ can blow up and depend exponentially on $\max_{\varrho\in \Gamma}\vectornorm{\varrho}$.
\end{abstract}

% keywords can be removed

\section{Introduction}
The constraint satisfaction problem (CSP)~\cite{Khanna} is a general computational problem in which an instance consists of a finite set of variables, $V$, a finite domain, $D$, and a set of constraints $C = \{c_1, \cdots, c_m\}$ where $c_i = \langle (v^i_1,\cdots, v^i_{n_i}), \varrho_i\rangle$, $v^i_j\in V$, $\varrho_i\subseteq D^{n_i}$. The goal of CSP is to assign every variable in $V$ to a value from $D$ in such a way that all the constraints are satisfied, or, in other words, an assignment of the tuple $(v^i_1,\cdots, v^i_{n_i})$ should be in $\varrho_i$.
The valued constraint satisfaction problem (VCSP)~\cite{COOPER2003311,Schiex} is a generalization of CSP where every constraint $c_i$ is a pair $\langle (v^i_1,\cdots, v^i_{n_i}), f_i\rangle$, $f_i:  D^{n_i}\rightarrow {\mathbb Z}$ and ${\mathbb Z}$ is a set of integers. The goal of VCSP is to find an assignment $h: V\rightarrow D$ that minimizes the energy function:
$$
E[h] = \sum_{i=1}^m f_i(h(v^i_1),\cdots, h(v^i_{n_i})).
$$
(V)CSP is a general framework that captures a large set of computational problems arising in propositional logic, combinatorial optimization, graph theory, artificial intelligence, scheduling, biology (protein folding), computer vision, etc. Since the satisfiability problem is a special case of them, CSP and VCSP are NP-hard in general.

From the 70s an active and fruitful research direction in theoretical computer science is to characterize restrictions on an instance that imply the tractability of the CSP. A natural restriction is to assume that the predicates (i.e. $\varrho_i$) in CSP, or the cost functions (i.e. $f_i$) in VCSP, are  from a prespecified set of predicates (functions), denoted by $\Gamma$, which is called the constraint language. The restriction defines the so-called fixed-template (V)CSP, denoted by ${\rm (V)CSP}(\Gamma)$, whose complexity was classified for all possible constraint languages~\cite{HELL199092,JEAVONS1998185,BJK,BulatovFOCS,Feder,ZhukDmitriy}. The next natural idea is to restrict the structure of constraints, by which we mean various graphs that can be associated with an instance, such as the primal graph, the dual graph, the incidence graph,  the constraint hypergraph, the constraint relational structure, the microstructure graph etc~\cite{DechterPearl,GOTTLOB2002579,GroheMarx,NewHybrid}. 

%Another promising idea is to study the algebraic structure of the so-called lifted language ${\mathcal L}_I$ that can be constructed from any instance $I=\left(V,D,C\right)$. 
%Lifted languages attract attention in the following context. For two ordered constraint languages $\Pi = \{\pi_1, \cdots, \pi_s\}$, $\pi_i\subseteq \Delta^{n_i}$, $\Gamma=\{\gamma_1, \cdots, \gamma_s\}$, $\gamma_i:  D^{n_i}\rightarrow {\mathbb Z}$, a restriction of VCSP,  called the VCSP with the input prototype, denoted ${\rm VCSP}(\Pi, \Gamma)$, is a search problem in which we are given two instances, $\left(V,D,C'\right)$ and $\left(V,D,C\right)$ of ${\rm CSP}(\Pi)$ and ${\rm VCSP}(\Gamma)$ respectively, such that $\langle (v_1,\cdots, v_{n_i}), \pi_i\rangle\in C' \Leftrightarrow \langle (v_1,\cdots, v_{n_i}), \gamma_i\rangle\in C$ and, additionally, we are given a solution $h: V\to \Delta$ for an instance $\left(V,D,C'\right)$. The goal of ${\rm VCSP}(\Pi, \Gamma)$ is to find the solution for $\left(V,D,C\right)$\footnote{The definition reminds of the promise CSP~\cite{Brakensiek,Jakub}, though the two problems are substantially different.}. In~\cite{HybridVCSPs} it was shown that ${\rm VCSP}({\mathcal L}_I)$ is equivalent to ${\rm VCSP}(\Pi, \Gamma)$, where $\Pi$ and $\Gamma$ are defined by $I$.
%The lifted language is a proof tool when dealing with fixed-template CSPs with additional restrictions on the constraint relational structures, namely, such restrictions that are closed under inverse homomorphisms.
Using algebraic techniques based on so-called lifted languages~\cite{HybridVCSPs,takhanov2017searching}, it was shown that the complexity of ${\rm VCSP}(\Gamma)$ does not change even if we add the restriction that the set $V$ is ordered, i.e. $V\subseteq {\mathbb Z}$, and for every constraint $c_i = \langle (v^i_1, \cdots, v^i_{n_i}), f_i \rangle$ we have $v^i_1 < \cdots < v^i_{n_i}$~\cite{Effectiveness}. This problem, denoted ${\rm (V)CSPO}(\Gamma)$, is called fixed-template (V)CSP with ordered variables. Thus, the computational complexities of ${\rm (V)CSP}(\Gamma)$ and ${\rm (V)CSPO}(\Gamma)$
are the same.

%\begin{definition}
%For a given template ${\boldsymbol \Gamma}$ and a relational structure ${\bf P}$, {\bf the CSP with input prototype} ${\bf P}$ is a problem, denoted $\textsc{CSP}^+_{{\bf P}}({\boldsymbol \Gamma})$, for which: a) an instance is a pair $({\bf R}, \chi)$ where ${\bf R}$ is a relation structure and $\chi:{\bf R}\to {\bf P}$ is a homomorphism; b) a goal is find a homomorphism $h: {\bf R}\to {\boldsymbol \Gamma}$. 
%\end{definition}z

A direct restriction of ${\rm (V)CSPO}(\Gamma)$ is the following problem, denoted ${\rm (V)CSPO}^\ast (\Gamma)$: let us assume $V\subseteq {\mathbb Z}$ and for every constraint $c_i = \langle (v^i_1, \cdots,  v^i_{n_i}), f_i\rangle$ we have $v^i_1+1=v^i_2, \cdots, v^i_{n_i-1}+1 = v^i_{n_i}$, i.e. the constraints can be applied only to a segment of adjacent variables. In the paper we will study the complexity of this problem for certain constraint languages. 

The specific case of ${\rm VCSPO}^\ast (\Gamma)$ is the minimization of a pattern-based potential.
Initially, the pattern-based potential was defined as the following energy function over $x\in D^n$:
\begin{equation}
E(x)=\sum_{\alpha\in\Pi}\sum_{\substack{[i,j]\subseteq[1,n]\\j-i+1=|\alpha|}}c^\alpha_{ij}\cdot[x_{i:j}=\alpha]
\label{eq:patternCRF}
\end{equation}
where $\Pi\subseteq  D^\ast = \bigcup_{i=0}^\infty D^i$ is a fixed set of non-empty words, $|\alpha|$ is the length of word $\alpha$, $x_{i:j} = x_i x_{i+1}\cdots x_{j}$, $c^\alpha_{ij}\in {\mathbb Z}$ and $[\cdot]$ is the {\em Iverson bracket}, i.e. $[{\bf True}]=1$ and $[{\bf False}]=0$. It can be shown that the minimization of a pattern-based potential is equivalent to ${\rm VCSPO}^\ast (\Gamma)$ where $\Gamma = \{f_{\alpha}| \alpha\in \Pi\}$ and $f_{\alpha}: D^{|\alpha|}\rightarrow {\mathbb Z}$, $f_{\alpha}(x) = [x=\alpha]$. Intuitively, such an energy potential models long-range interactions for selected subsequences of labels~\cite{Ye:NIPS09, TakhanovKolm}. Pattern-based potentials have been applied to handwritten character recognition, identification of named entities from text~\cite{Ye:NIPS09}, optical character recognition~\cite{Qian:ICML09} and  language modeling~\cite{PatternsVersus, takhanov2014combining}.
In~\cite{Ye:NIPS09, TakhanovKolm} it was shown that the potential~\eqref{eq:patternCRF} can be minimized in time ${\mathcal O}(n |\Pi^\ast| \min\{|D|, \log(\max_{\alpha\in \Pi}|\alpha|+1)\})$ where $\Pi^\ast = \{\alpha| \alpha\beta\in \Pi\}$ is the prefix closure of the set $\Pi$.

A direct generalization of the pattern-based potential is the following energy function over $x\in D^n$:
\begin{equation}
E(x)=\sum_{\varrho\in\Gamma}\sum_{\substack{[i,j]\subseteq[1,n]\\j-i+1=\vectornorm{\varrho}}}f^\varrho_{ij}\cdot[x_{i:j}\in \varrho]
\label{eq:patternCRFgen}
\end{equation}
where $\Gamma\subseteq  \bigcup_{i=0}^\infty 2^{D^i}$ is a fixed set of predicates over $D$, $\vectornorm{\varrho}$ is the arity of the predicate $\varrho$ and $f^\varrho_{ij}\in {\mathbb Z}$. Analogous to the previous potential, the minimization of the generalized pattern-based potential~\eqref{eq:patternCRFgen} is equivalent to ${\rm VCSPO}^\ast (\Gamma')$ where $\Gamma' = \{f_{\varrho}| \varrho\in \Gamma\}$ and $f_{\varrho}: D^{\vectornorm{\varrho}}\rightarrow {\mathbb Z}$, $f_{\varrho}(x) = [x\in \varrho]$.

The minimization of~\eqref{eq:patternCRFgen} can be easily reduced to the minimization of~\eqref{eq:patternCRF} by simply plugging in $\sum_{\alpha\in \varrho} [x_{i:j}=\alpha]$ instead of $[x_{i:j}\in \varrho]$ and representing~\eqref{eq:patternCRFgen} as the pattern-based potential. This strategy leads to an algorithm that minimizes the generalized pattern-based potential~\eqref{eq:patternCRFgen}
in time ${\mathcal O}(n |\Gamma^\triangleright| \cdot \min\{|D|, \log(l_{\max}+1)\})$ where $\Gamma^\triangleright = \{(x_1, \cdots, x_i)| (x_1, \cdots, x_{\|\varrho\|})\in \varrho\in \Gamma, 1\leq i\leq \|\varrho\|\}$ is a prefix closure of all tuples in $\bigcup_{\varrho\in \Gamma}\varrho$ and $l_{\max}=\max_{\varrho\in \Gamma}\vectornorm{\varrho}$.

In practice the value of $|\Gamma^\triangleright|$ is the main factor that limits the applicability of this reduction, because $|\varrho|$, for $\varrho\in \Gamma$, in natural applications can be very large~\cite{Qian:ICML09}. E.g., if any $\varrho\in \Gamma$ can be represented as $\varrho = \Omega_1\times \cdots \times\Omega_{\vectornorm{\varrho}}$ (such languages $\Gamma$ are called {\em simple}), then $|\varrho| = |\Omega_1|\times \cdots \times |\Omega_{\vectornorm{\varrho}}|$ exponentially depends on $\vectornorm{\varrho}$. Thus, designing an algorithm for the minimization of~\eqref{eq:patternCRFgen} and avoiding an exponential dependance on $l_{\max}$ becomes an actual problem. Note that, for the general case, such a dependance is unavoidable. Indeed, suppose that $D=\{0,1\}$ and there is an algorithm with the complexity $\mathcal{O}({\rm poly}(n, f(l_{\max})))$ for the minimization of~\eqref{eq:patternCRFgen},  where $f$ is a subexponential function. Since $l_{\max}\leq n$, this would imply that Weighted Max SAT is solvable in a subexponential time, which contradicts to the exponential time hypothesis~\cite{Impagliazzo}.

Let us briefly outline our results.
Given the language $\Gamma$, first we describe a closure operator $\overline{\Gamma^{\cap}}$ that adds new predicates to $\Gamma$, with the inclusion property $\overline{\Gamma^{\cap}}\supseteq \Gamma$. It turns out that the complexity of the minimization of~\eqref{eq:patternCRFgen} depends on the cardinality of $\overline{\Gamma^{\cap}}$, rather than on the cardinality of $\Gamma^\triangleright$. Moreover, as the examples in Section~\ref{Examples} demonstrate, for some natural languages $\Gamma$ we have $\overline{\Gamma^{\cap}} = \Gamma$, and all the more $|\overline{\Gamma^{\cap}}| = |\Gamma| \ll |\Gamma^\triangleright|$.
We present algorithms for the following problems:
\begin{itemize}
\item The minimization of~\eqref{eq:patternCRFgen} for non-positive $f^\varrho_{ij}$: a) in time ${\mathcal O}(n |\overline{\Gamma^{\cap}}|^2)$ for a general language $\Gamma$, b) in time ${\mathcal O} \big(n |\overline{\Gamma^\cap}|\cdot |D|  \cdot l_{\max}^2\big) $ for a simple language $\Gamma$;
\item The minimization of~\eqref{eq:patternCRFgen} for arbitrary $f^\varrho_{ij}$ in time ${\mathcal O}(n |D|^2 \cdot |\overline{\Gamma^{\cap}}|^2)$ for a simple language $\Gamma$;
\item The computation of the partition function $\sum_{x\in D^n}e^{-E(x)}$ in time ${\mathcal O}(n |D|^2 \cdot |\overline{\Gamma^{\cap}}|^2)$ for a simple language $\Gamma$;
\end{itemize}

Note that both the first two tasks and the last one can be uniformly treated as the computation of a partition function over a general semiring~\cite{Bistarelli1999}. Our general algorithm is described in this framework in Section~\ref{generalAlg}.
\section{Related work}
Computational aspects of generalized pattern-based potentials were first formulated and addressed in machine learning, especially in Conditional Random Fields (CRF) research communities. A direct generalization of the linear chain CRFs based on 2-grams~\cite{Lafferty}, with a purpose to include interactions between non-adjacent labels, leads to CRFs with higher-order terms. A key problem with a general higher-order CRF is the NP-hardness of the main three tasks~\cite{Istrail}: the inference, the maximum a posteriori probability (MAP) estimation and the computation of a partition function.
Specific methods to handle the hardness of long-range interactions include works on semi-Markov CRFs~\cite{Sarawagi}, a reduction of the MAP for the sparse high-order CRFs to quadratic functions~\cite{Rother}, the direct dynamic programming approach to sparse high-order CRFs~\cite{QianXian}, the dual decomposition scheme for high-order MRFs~\cite{Komodakis}. Pattern-based potentials were studied in~\cite{Komodakis} and it was shown that a simple message passing algorithm calculates the MAP in linear time when all high-order terms of such a potential are non-positive. In~\cite{cuong14a} it was shown that all three tasks for pattern-based potentials can be solved by linear algorithms.  Refined versions of these algorithms and an efficient sampling technique were described in~\cite{TakhanovKolm}. Recently, new techniques to training pattern-based CRFs were suggested~\cite{vieira-etal-2016-speed,lavergne-yvon-2017-learning,martins-etal-2011-structured,PatternsVersus} in which a sparse set of patterns is represented in a finite-state automaton. Applications of pattern-based CRFs include signal
reconstruction, image denoising, binary image segmentation and stereo matching~\cite{Komodakis}, named entity recognition~\cite{QianXian},  protein dihedral angles prediction~\cite{zhalgas},  contour detection~\cite{Felzenszwalb} and many others.

%{\bf Named entity recognition (NER).} The NER task is a sequence labeling task in which a set of input sequences is $\mathcal{X}=\mathcal{W}^\ast$ where $\mathcal{W}$ is a finite set of words of some natural language (e.g. english) and a set of output sequences is $\mathcal{Y}=\{B,I,O\}^\ast$. We are given a training set $\mathcal{T} \subseteq \mathcal{X}\times \mathcal{Y}$. For any $(x,y)\in \mathcal{T}$ we have $x=w_1\cdots w_l, w_i\in \mathcal{W}$ and $y=o_1\cdots o_l, o_i\in \{B,I,O\}$ and we interpret $o_i$ as a tag of the corresponding word $w_i$. The tag $B$ marks a beginning of person's name, $I$ marks an inside word of a person's name and $B$ marks everything else. For example, the sentence "Today professor Mark Wilson missed the class" is labeled as $OBIIOOO$.
\subsection{A connection with a pattern-based potential on a grid}
The mentioned paper of Komodakis and Paragios~\cite{Komodakis} also describes an algorithm that solves the minimization of a pattern-based potential on a grid using an oracle that solves the task~\eqref{eq:patternCRF}. A pattern-based potential on a grid $[m]\times [n]$ is an energy function $E: D^{m\times n}\to {\mathbb R}$ that is defined by
\begin{equation*}
\begin{split}
E([x_{ij}]_{i\in [m], j\in [n]}) = \sum_{i\in [m],j\in [n]}u_{ij}(x_{ij})+\\
\sum_{a\in [m],b\in [n]: a+k-1\leq m, b+l-1\leq n}\psi_{ab}(x_{a:a+k-1, b:b+l-1}),
\end{split}
\end{equation*}
where $x_{ij}\in D$, $x_{a:a+k-1, b:b+l-1} = [x_{ij}]_{a\leq i\leq a+k-1,b\leq j\leq b+l-1}$ and $k,l\in {\mathbb N}$ are fixed small numbers (e.g. $k=l=3$). Unary terms $u_{ij}: D\to {\mathbb R}$ are arbitrary, but the higher-order terms $\psi_{ab}: D^{k\times l}\to {\mathbb R}$ are assumed to satisfy the assumption that $\{x\in D^{k\times l}\mid \psi_{ab}(x)\ne 0\}$ is of moderate size (the so-called sparsity assumption).
The dual decomposition scheme of Komodakis and Paragios reduces the minimization of the latter energy (then, this task is called a master task) to so-called horizontal slave tasks of the following kind
\begin{equation*}
\begin{split}
E_s([x_{ij}]_{i\in [k], j\in [n]}) = \sum_{i\in [k], j\in [n]}v_{ij}(x_{ij})+
\sum_{b\in [n]: b+l-1\leq n}\psi_{b}(x_{1:k, b:b+l-1}),
\end{split}
\end{equation*}
where $v_{ij}$ are arbitrary unary functions and $\psi_{b}$ are higher order terms (vertical slave tasks have analogous structure, see Figure~\ref{master}). Then, after defining $D' = D^k$ and treating the matrix $[x_{ij}]_{i\in [k], j\in [n]}$ as a string of column vectors  ${\mathbf x}_1, \cdots, {\mathbf x}_n$, ${\mathbf x}_i\in D'$, one obtains an equivalent minimization task on a chain of variables
\begin{equation*}
\begin{split}
E_s({\mathbf x}_1, \cdots, {\mathbf x}_n) = \sum_{i\in [n]}v_{i}({\mathbf x}_i)+
\sum_{b\in [n]: b+l-1\leq n}\psi_{b}({\mathbf x}_b, \cdots, {\mathbf x}_{b+l-1}),
\end{split}
\end{equation*}
where $\{({\mathbf x}_1, \cdots, {\mathbf x}_l)\in (D')^{l}\mid \psi_{b}({\mathbf x}_1, \cdots, {\mathbf x}_l)\ne 0\}$ is still sparse. The dual decomposition approach allows to handle the case where higher-order terms can be applied to rectangles $x_{a:a+k-1, b:b+l-1}$ with fixed sizes $k,l\in {\mathbb N}$. Though, if a master task includes higher-order terms with different sizes of $k$ and $l$, or non-rectangular higher-order terms, like e.g. $\psi(x_{11},x_{22}, x_{33})$, then higher-order terms of slave tasks will violate the sparsity assumption. Indeed, let $k=l$ and $\psi(x_{11},x_{22}, \cdots, x_{kk}) = [(x_{11}, \cdots, x_{kk}) = (1, \cdots, 1)]$ be a higher-order term in a master task with variables $[x_{ij}]_{i\in [m], j\in [n]}$. Let ${\mathbf x}_j = [x_{1j}, \cdots, x_{kj}]^T$, $j\in [n]$. Then the relation $\varrho = \{({\mathbf x}_1, \cdots, {\mathbf x}_k)\mid (x_{11}, \cdots, x_{kk}) = (1, \cdots, 1)\}$ does not satisfy the sparsity assumption due to $|\varrho| = |D|^{k^2-k}$ (if $k>4$). But,
from the definition of a slave task on variables $[k]\times [n]$ it follows that it contains the higher-order term $\psi'({\mathbf x}_1, \cdots, {\mathbf x}_k) = [({\mathbf x}_1, \cdots, {\mathbf x}_k)\in \varrho]$.

Thus, to generalize the dual decomposition scheme to a pattern-based potential on a grid with arbitrary (possibly non-rectangular) localized cliques, one have to deal with more general pattern-based potentials on a chain.
One of motivations to study generalized pattern-based potentials is to answer the latter problem.
\begin{figure}

  \begin{center}
\includegraphics[width=10cm]{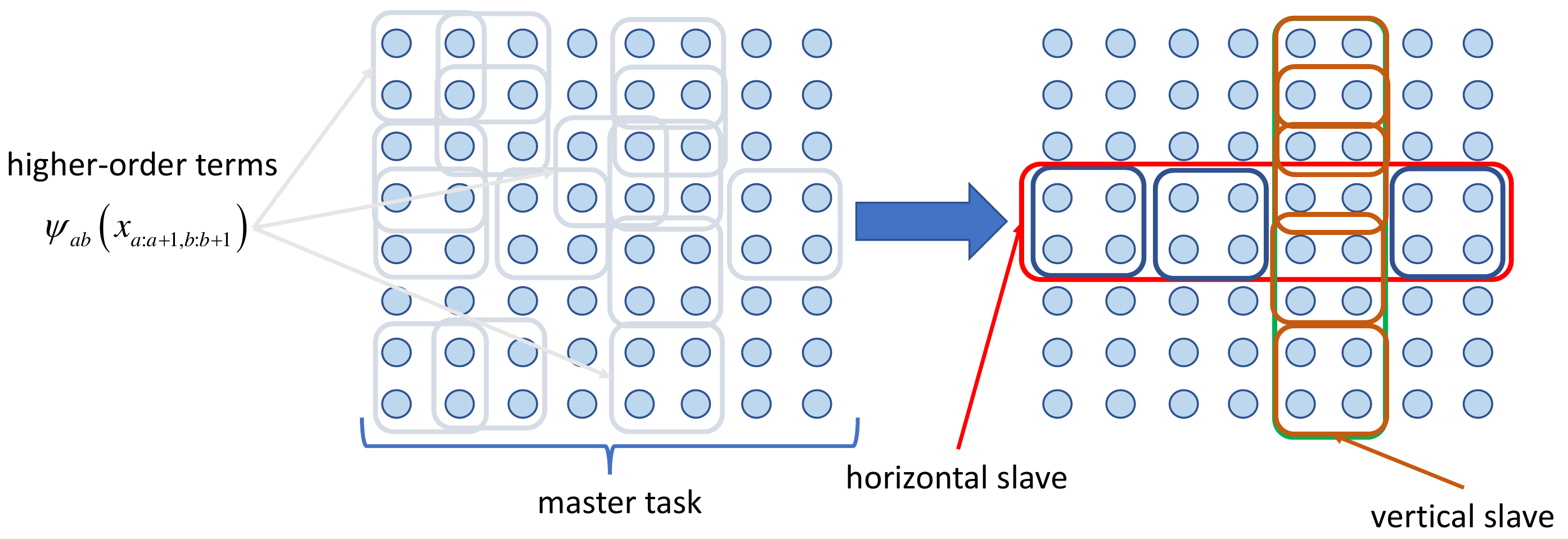}
  \end{center}
\caption{A pattern-based potential on a grid for $k=l=2$. Slave tasks are formed by a set of variables $\{i,i+1\}\times [n]$ (horizontal slave) or $[m]\times \{j,j+1\}$ (vertical slave). A set of soft constraints imposed on a set of variables of a slave task are added to the task (with unary constraints being divided by the number of slaves they participate in).}\label{master}
\end{figure}

\subsection{Generalized pattern-based potentials as bounded treewidth VCSPs}
The minimization of a generalized pattern-based potential~\eqref{eq:patternCRFgen} is a special case of Valued CSP. A dual graph $\mathcal{G} = (V,E)$ of this VCSP instance is a graph with the set of vertices $V=[n]$ and a set of edges $\{i,j\}, i<j$ such that variables $x_i$ and $x_j$ participate together in a scope of some valued constraint (i.e. $j-i+1\leq \|\varrho\|$ for some $\varrho\in \Gamma$). By the definition of the potential~\eqref{eq:patternCRFgen}, the bandwidth of the graph $\mathcal{G} $ does not exceed $l_{\max}-1$ where $l_{\max}=\max_{\varrho\in \Gamma}\|\varrho\|$. Since the treewidth is bounded by the bandwidth, we conclude that the treewidth of the dual graph $\mathcal{G}$ is bounded by $l_{\max}-1$. VCSPs whose dual constraint graphs have a bounded treewidth were studied in a line of papers~\cite{Terrioux,Givry,Ndiaye,Scarcello} and in the PhD Thesis~\cite{Farnqvist}. Backtracking techniques applied to~\eqref{eq:patternCRFgen} achieve the time complexity $\mathcal{O}(n^2 |\Gamma|\cdot |D|^{l_{\max}}l^2_{\max} \log |D|)$~\cite{Terrioux}, while reducing the problem to Valued CSP with an acyclic hypergraph requires $\mathcal{O}(n |D|^{l_{\max}}\log |D|^{l_{\max}})$ time~\cite{Scarcello}. In all these algorithms, the factor $|D|^{l_{\max}}$ in time complexity is a crucial obstacle for practical applications. This factor appears due to the fact that generic dynamic programming algorithms do not take into account the specificity of the language $\Gamma$. To our knowledge the only algorithm for~\eqref{eq:patternCRFgen} that explicitly depends on $\Gamma$ is the reduction to the pattern-based potential minimization that we described in the introduction. Recall that its time complexity is  ${\mathcal O}(n |\Gamma^\triangleright| \cdot \min\{|D|, \log(l_{\max}+1)\})$ where $\Gamma^\triangleright$ is a prefix closure of all tuples in $\bigcup_{\varrho\in \Gamma}\varrho$. Therefore, we consider the latter time complexity as a baseline theoretical complexity for~\eqref{eq:patternCRFgen} and we will compare new algorithms with this baseline. All time compexities of mentioned algorithms are collected in Table~\ref{compl}.
\begin{table}
\begin{center}
\begin{tabular}{ |c|c| }
 \hline
 Algorithm & Time complexity \\ 
 \hline
BTD$_{val}$, backtracking~\cite{Terrioux} & $\mathcal{O}(n^2 |\Gamma|\cdot |D|^{l_{\max}}l^2_{\max} \log |D|)$ \\
\hline 
tree decomposition~\cite{Scarcello} & $\mathcal{O}(n |D|^{l_{\max}}\log |D|^{l_{\max}})$ \\
\hline
pattern-based~\cite{TakhanovKolm} & ${\mathcal O}(n |\Gamma^\triangleright| \cdot \min\{|D|, \log(l_{\max}+1)\})$ \\
\hline
Algorithm~\ref{alg:DP2} (simple $\Gamma$) & ${\mathcal O}(n |D|^2 \cdot |\overline{\Gamma^{\cap}}|^2)$ \\
\hline
Algorithm~\ref{alg:DP2} (gen. $\Gamma$) & ${\mathcal O} (n |\Gamma^\ast|^2)$ \\
\hline
Algorithm~\ref{alg:DP'} (simple $\Gamma$, $f^{\varrho}_{ij}\leq 0$) & ${\mathcal O} \big(n |\overline{\Gamma^\cap}|\cdot |D|  l^2_{\max}\big) $ \\
\hline
Algorithm~\ref{alg:DP'} (gen. $\Gamma$, $f^{\varrho}_{ij}\leq 0$) & ${\mathcal O}(n |\overline{\Gamma^{\cap}}|^2)$  \\
\hline
\end{tabular}
\end{center}
\caption{\centering Time complexities of algorithms.} \label{compl}
\end{table}
\section{Generalized pattern-based potentials}

%Throughout this section $x=x_1\ldots x_n$ denotes a sequence of $n$ variables (or their values), where $n$ is fixed.
Let us first introduce some notations.
For a predicate $\varrho\subseteq D^k$
we denote $\vectornorm{\varrho}=k$, and $\varrho_{i:j}=\left\{(x_i,...,x_j)|(x_1,...,x_k)\in \varrho\right\}$. Sometimes we use ${\rm proj}_i \varrho$ instead of $\varrho_{i:i}$. By a string $x$ over $D$ we mean any element of $\bigcup_{i=1}^\infty D^i$ or an empty string $\lambda$. A string $x=(x_1,\cdots, x_n)\in D^n$ is simply denoted by $x_1\cdots x_n$. If $x=(x_1,\cdots, x_n)$ then $x_{i:j}=x_i\ldots x_j$ denotes a substring of $x$ from $i$ to $j$. By $[l]$ we denote the set $\{1,\cdots, l\}$.

Suppose that $D$ is finite and we are given a set $\Gamma$ of predicates over $D$, i.e. each predicate $\varrho\in\Gamma$ is a subset of $D^{\vectornorm{\varrho}}$.

\begin{mdframed}
By $PPM\left(\Gamma\right)$ (pattern-based potential minimization) we denote a minimization task with:

{\bf Instance}: a natural number $n$ and rational values $\left\{c^\varrho_{ij}\right\}_{(\varrho,i,j)\in\Lambda}$ where $\Lambda$ is the set of triplets $(\varrho,i,j)$ with $\varrho\in\Gamma$ and $j-i=\vectornorm{\varrho}-1$, $[i,j]\subseteq [1,n]$.

{\bf Solution}: A word $w$ over $D$ of length $n$ that minimizes the objective.

{\bf Objective}: A function given by the formula:
\begin{equation}
F\left( x \right) = \sum\limits_{\left( {\varrho ,i,j} \right) \in\Lambda} {f_{ij}^\varrho \left( {{x_{i:j}}} \right)}
\qquad\mbox{where}\qquad
f_{ij}^\varrho \left( y \right) =
\begin{cases}
c_{ij}^\varrho & \mbox{if }y\in\varrho \\
0 & \mbox{otherwise}
\end{cases}
\label{eq:f}
\end{equation}
over words $x$ of length $n$.
\end{mdframed}

When we require input weights $c^\varrho_{ij}$ to be non-positive, we will use the notation $PPM^{-}\left(\Gamma\right)$.

Let us also introduce the following generalization of the previous problem.
Let $\mathcal{R} = (R,\oplus,\otimes)$ be a commutative semiring with elements $\mathbb{O},\mathds{1}\in R$
which are identities for $\oplus$ and $\otimes$ respectively.

\begin{mdframed}
By $PPSS\left(\Gamma, \mathcal{R}\right)$ (pattern-based potential summation over semiring) we denote a summation task, where

{\bf Instance}: a natural number $n$ and values $\left\{c^\varrho_{ij}\right\}_{(\varrho,i,j)\in\Lambda}\subseteq R$.

 Define the cost of labeling $x\in D^n$ via
\begin{equation}
\!\!\!\!\!\!F( x ) =\!\!\!\!\! \bigotimes_{\left( {\varrho ,i,j} \right) \in\Lambda} \!\!\! {f_{ij}^\varrho \left( {{x_{i:j}}} \right)}
,\;\;
f_{ij}^\varrho ( y ) =
\begin{cases}
c_{ij}^\varrho & \mbox{if }y\in\varrho \\
\mathds{1} & \mbox{otherwise}
\end{cases}\!\!\!\!\!\!
\label{eq:F}
\end{equation}

{\bf Solution}: Our goal is to compute\!\!\!\!\!\!
\begin{equation}
Z=\bigoplus_{x\in D^n} F(x)\label{eq:Fsum}
\end{equation}
\end{mdframed}

\begin{remark}
Note that when $\mathcal{R}=(\mathbb Q,\min, +)$, $PPSS\left(\Gamma, \mathcal{R}\right)$ is equivalent to $PPM\left(\Gamma\right)$. When $\mathcal{R}=(\mathbb Q, +, \times)$, and $f_{ij}^\varrho(x_{i:j}) = e^{-c^\varrho_{ij}[x_{i:j}\in \varrho]}$, then $PPSS\left(\Gamma, \mathcal{R}\right)$ is equivalent to the computation of
$$
\sum_{x\in D^n} e^{-E(x)}.
$$ 
where $E(x) = \sum_{\left( {\varrho ,i,j} \right) \in\Lambda} c^\varrho_{ij}[x_{i:j}\in \varrho]$ is a generalized pattern-based energy.
We will be interested in the computational complexity of both $PPM\left(\Gamma\right)$ and $PPSS\left(\Gamma, \mathcal{R}\right)$.
\end{remark}

Note that if $\alpha = D^{k}\times \alpha'$ and $\alpha,\alpha'\in\Gamma$ then
we can eliminate $\alpha$ from $\Gamma$, i.e. delete the longer predicate from $\Gamma$, and update the weight $c_{i+k, j}^{\alpha'}\leftarrow c_{i+k, j}^{\alpha'}\otimes c_{i, j}^{\alpha}$ correspondingly.
By construction, such an elimination does not change the function that is determined by \eqref{eq:f} or \eqref{eq:F}. The same holds in the case in which  $\alpha = \alpha'\times D^{k}$ and $\alpha,\alpha'\in\Gamma$.

\section{Closures of the constraint language $\Gamma$}\label{closures}

For predicates $\alpha,\beta$ over $D$, denote
$$
\alpha | \beta = \begin{cases}
{\left\{(x_1,...,x_{\vectornorm{\alpha}})\in \alpha|(x_{\vectornorm{\alpha}-\vectornorm{\beta}+1},...,x_{\vectornorm{\alpha}})\in \beta\right\},{\rm{\,\,if\,\,}} \vectornorm{\alpha}>\vectornorm{\beta}} \\
  {\left\{(x_1,...,x_{\vectornorm{\beta}})\in \beta|(x_{\vectornorm{\beta}-\vectornorm{\alpha}+1},...,x_{\vectornorm{\beta}})\in \alpha\right\},{\rm{\,\,if\,\,otherwise\,\,}}} \\
\end{cases} 
$$
It is easy to see that $\varrho^1 | (\varrho^2 | \varrho^3) = (\varrho^1 | \varrho^2) | \varrho^3$ and $\varrho^1 | \varrho^2 = \varrho^2 | \varrho^1$, i.e. $|$ is an associative and a commutative operation.
If $\vectornorm{\varrho}>0$, let us denote $\varrho^{-} = \left\{(x_1,...,x_{\vectornorm{\varrho}-1})|(x_1,...,x_{\vectornorm{\varrho}})\in\varrho\right\}$. If $\vectornorm{\varrho}=1$ and $\lambda$ is an empty word, then the predicate $\varrho^{-} = \{\lambda\}$ is denoted by $\varepsilon$.

The definition of $\alpha|\beta$ is motivated by the following observation. For any predicate $\alpha$ over $D$ let us denote $*\alpha$ the set of words $w$ over $D$ such that $w_{|w|-\vectornorm{\alpha}+1:|w|}\in\alpha$.
Then for any two predicates $\alpha$ and $\beta$ we have $(\ast\alpha)\cap(\ast\beta) = \ast(\alpha|\beta)$.

\begin{definition}
Let $\Gamma$ be a finite set of predicates over $D$. 

\begin{enumerate}

\item The intersection of all supersets $\widehat\Gamma \supseteq \Gamma$, such that
\begin{itemize}
\item $\varrho\in \widehat\Gamma\rightarrow \varrho^{-}\in \widehat\Gamma$.
\end{itemize}
is called the prefix closure of $\Gamma$, and is denoted as $\overline{\Gamma}$. If $\Gamma = \overline{\Gamma}$, then  $\Gamma$ is called prefix closed.

\item The intersection of all supersets $\widehat\Gamma \supseteq \Gamma$, such that
\begin{itemize}
\item $\varrho^1\in \widehat\Gamma,\varrho^2\in \Gamma\rightarrow \varrho^1 | \varrho^2 \in \widehat\Gamma$
\end{itemize}
is called the intersection closure of $\Gamma$, and is denoted as $\Gamma^\cap$. If $\Gamma = \Gamma^\cap$, then  $\Gamma$ is called intersection closed.

\item The intersection of all supersets $\widehat\Gamma \supseteq \Gamma$, such that
\begin{itemize}
\item $\varrho\in \widehat\Gamma, a\in D\rightarrow \varrho^{-}\times \{a\}\in \widehat\Gamma$
\end{itemize}
is called the singleton suffix closure of $\Gamma$, and is denoted as $\Gamma^\circ$. If $\Gamma = \Gamma^\circ$, then  $\Gamma$ is called singleton suffix closed.

\item The intersection of all prefix closed and intersection closed supersets $\widehat\Gamma \supseteq \Gamma$ is denoted as $\overline{\Gamma^{\cap}}$.

\item The intersection of all prefix closed, intersection closed and singleton suffix closed supersets $\widehat\Gamma \supseteq \Gamma$ is called the closure of $\Gamma$ and is denoted as $\Gamma^\ast$.
\end{enumerate}

\end{definition}

\begin{definition} If each predicate $\varrho\in\Gamma$
is a direct product of subsets of $D$, i.e. $\varrho = \Omega_1\times \cdots \times \Omega_{\vectornorm{\varrho}}$, then $\Gamma$ is called simple.
\end{definition} 

\begin{theorem}\label{simple}
For a simple language $\Gamma$, we have $\Gamma^\ast = \left(\overline{\Gamma^{\cap}}\right)^\circ$.
\end{theorem}

A proof of the theorem is straightforward from the following lemma.

\begin{lemma} 
If $\Gamma$ is simple, prefix closed and intersection closed, then $\Gamma^\circ$ is also simple, prefix closed and intersection closed.
\end{lemma}
\begin{proof} From the definition of the singleton suffix closure we have
$\Gamma^\circ = \{\varrho^{-}\times \{a\} | \varrho\in \Gamma, a\in D\}\cup \Gamma$. Obviously, $\Gamma^\circ$ is simple and prefix closed. It is easy to see that for any $\varrho_1, \varrho_2\in \Gamma$ we have a) $(\varrho^{-}_1\times \{a\})| (\varrho^{-}_2\times \{a\}) = (\varrho^{-}_1| \varrho^{-}_2)\times \{a\}$, b) $(\varrho^{-}_1\times \{a\})| (\varrho^{-}_2\times \{b\}) = \emptyset$, for $a\ne b$, c) $\varrho^{-}_1\times \{a\}| \varrho_2  = (\varrho^{-}_1| \varrho^{-}_2)\times \{a\}$ if $a\in {\rm proj}_{\vectornorm{\varrho_2}}\varrho_2$, d) $\varrho^{-}_1\times \{a\}| \varrho_2  = \emptyset$ if $a\notin {\rm proj}_{\vectornorm{\varrho_2}}\varrho_2$. Thus, $\Gamma^\circ$ is  intersection closed.
\end{proof}

Thus, Theorem~\ref{simple} gives that for simple languages the cardinality of $\Gamma^\ast$ cannot be substantially larger than the cardinality of $|\overline{\Gamma^{\cap}}|$:
\begin{equation}\label{nextsimple}
|\Gamma^\ast| \leq (|D|+1)\times |\overline{\Gamma^{\cap}}|.
\end{equation}

As we will show in Section~\ref{generalAlg}, the complexity of $PPSS\left(\Gamma, \mathcal{R}\right)$ depends on cardinality of $\Gamma^{*}$, and therefore, for simple languages, on the cardinality of $\overline{\Gamma^{\cap}}$.

It can be checked that if $\Gamma$ is simple, then the set $\{\varrho^1_{1:i_1} | \varrho^2_{1:i_2} | \cdots | \varrho^k_{1:i_k}\,\, \mid\,\, \varrho^i\in \Gamma, i\in [k], k\in {\mathbb N}, i_s\in [\|\varrho^s\|]\}$ is intersection closed and prefix closed. The minimality of the latter superset of $\Gamma$ is also obvious. Thus,
$\overline{\Gamma^{\cap}}$ equals a set of all predicates that can be given with an expression $\varrho^1_{1:i_1} | \varrho^2_{1:i_2} | \cdots | \varrho^k_{1:i_k}$ for $\varrho^i\in \Gamma, k\in {\mathbb N}$. 
Thus, a trivial bound on $|\overline{\Gamma^{\cap}}|$ is the following:
$$
|\overline{\Gamma^{\cap}}|\leq 2^{|\Pi|}
$$
where $\Pi = \{\varrho_{1:i}| \varrho\in \Gamma, i\in [\vectornorm{\varrho}]\}$. This bound is pessimistic. In practice we often have moderate values of $|\overline{\Gamma^{\cap}}|$. 

\begin{example} Let $D = [n]$ and $\varrho = \{(1,x,2,y,3)| x,y\in D\}$, $\Gamma = \{\varrho\}$.
One can check that 
\begin{equation*}
\begin{split}
\overline{\Gamma^{\cap}} = \{\emptyset, \{(1)\}, \{(1,1)\}, \{(1,1,2)\}, \{(1,1,2,2)\}, \{(1,1,2,2,3)\}, \{(1,x)\mid x\in D\}, \\ \{(1,x,2)\mid x\in D\},
\{(1,x,2,1)\mid x\in D\}, \{(1,x,2,y)\mid x,y\in D\},\\ 
\{(1,x,2,1,3)\mid x\in D\}, \{(1,x,2,y,3)\mid x,y\in D\}
\}.
\end{split}
\end{equation*}
Thus, $|\overline{\Gamma^{\cap}}|=12$. Note that $|\overline{\Gamma^{\cap}}|$ does not depend on $|D|$ in that case, but $\sum_{\alpha\in\Gamma}|\alpha| = |\varrho| = |D|^2$.
\end{example}

Other examples of a simple languages (motivated by applications), for which the size of $\overline{\Gamma^{\cap}}$ is moderate, are given in Section~\ref{Examples}.

Before proceeding, let us introduce a partial order on a set of predicates. 
\begin{definition}
By $\alpha \geq \beta$ we denote $\ast\alpha \supseteq \ast\beta$. If
both $\alpha$ and $\beta$ cannot be decomposed as $D\times (\cdots)$, then, equivalently, $\alpha \geq \beta$ if and only if $\vectornorm{\alpha} \leq \vectornorm{\beta}$ and $\beta_{\vectornorm{\beta}-\vectornorm{\alpha}+1:\vectornorm{\beta}}\subseteq\alpha$. If $\alpha \geq \beta,\alpha \ne \beta$ we simply write $\alpha>\beta$. The partial order $\geq$, restricted to a set $\Gamma^\ast$, is denoted by $(\Gamma^\ast, \geq)$. A partial order $(\overline{\Gamma^{\cap}}, \geq)$ is defined analogously.
\end{definition}

\begin{definition}
By  $H[\Gamma^\ast]$ we denote a directed graph with the set of vertices $\Gamma^{*}$ and the edge set $E_{H}[\Gamma^\ast] = \big\{(\alpha,\beta)|\alpha > \beta, \mbox{there\,\,is\,\,no\,\,} \gamma\in \Gamma^\ast\mbox{\,\,s.t.\,\,}\alpha > \gamma>\beta\big\}$.
Obviously, $E_{H}[\Gamma^\ast]$ is the Hasse diagram of the partial order $(\Gamma^\ast, \geq)$. Analogously, $H[\overline{\Gamma^{\cap}}] = (\overline{\Gamma^{\cap}}, E_H[\overline{\Gamma^{\cap}}])$ is the Hasse diagram for $(\overline{\Gamma^{\cap}}, \geq)$.
\end{definition}
In Section~\ref{nonpositive}, we describe an algorithm solving $PPM\left(\Gamma\right)$ in the case of non-positive $c^\varrho_{ij}$ that is based on dynamic programming and has the time complexity ${\mathcal O} (n |E_{H}[\overline{\Gamma^{\cap}}]|)$. 
In Section~\ref{generalAlg}, we describe a general algorithm solving $PPSS\left(\Gamma\right)$ in time ${\mathcal O} \left(n |E_H[\Gamma^\ast]|\right)$. 
A proof of the following theorem can be found in Section~\ref{bounds}.
\begin{theorem}\label{basicbound} If $\Gamma$ is simple, then
$|E_H[\overline{\Gamma^\cap}]| \leq \frac{1}{2}\sum_{\alpha\in \overline{\Gamma^\cap}}(\vectornorm{\alpha} ^2+\vectornorm{\alpha} )\cdot |D|$.
\end{theorem}

From Theorem~\ref{basicbound} we conclude that, in the case of a simple language, complexities of both algorithms depend polynomially on the cardinality of $\overline{\Gamma^{\cap}}$.

\section{Algorithm for $PPM\left(\Gamma\right)$ with nonpositive $c^\varrho_{ij}$}\label{nonpositive}

%To simplify the presentation,
We define the ``partial'' cost of a sequence $x$ with length $s$ as
\begin{equation}
F_{s}(x) \triangleq \sum\limits_{\left( {\varrho ,i,j} \right)\in\Lambda, j\leq s} {f_{ij}^\varrho \left( {{x_{i:j}}} \right)}
\end{equation}
Note that $F(x)=F_n(x)$ for $|x|=n$.
%To handle boundary effects, it will be convenient to define $F(\ast x)\triangleq F(x)$ for $|x|=n$ and $F_s(\ast x)\triangleq F_s(x)$ for $|x|=s$.

In our algorithm we will compute the values $$W_{s}(\alpha) \triangleq \min\limits_{|x|=s,x\in(\ast\alpha)} F_{s} (x),\alpha\in \overline{\Gamma^{\cap}}.$$

\begin{theorem}\label{negative} If $c^\varrho_{ij}\leq 0$ for all $(\varrho,i,j)\in \Lambda$, then we have
\begin{equation}
\label{eq:alg}
W_s(\alpha) = \min\left\{\min\limits_{(\alpha,\beta)\in E_{H}[\overline{\Gamma^{\cap}}]} W_s(\beta), W_{s-1}(\alpha^-)+\phi_s(\alpha)\right\}
\end{equation}
where $\phi_s(\alpha)\triangleq\sum\limits_{\beta\in\Gamma:\beta\geq\alpha, (\beta, s-\|\beta\|+1, s) \in \Lambda }c^{\beta}_{s-\vectornorm{\beta}+1,s}$.
\end{theorem}
\begin{proof}
The inequality $W_s(\alpha) \leq \min\limits_{(\alpha,\beta)\in E_{H}[\overline{\Gamma^{\cap}}]} W_s(\beta)$ holds, because:
$$
\min\limits_{|x|=s,x\in(\ast\alpha)} F_{s} (x) \leq \min\limits_{|x|=s,x\in(\ast\beta)} F_{s} (x)
$$
for any $\beta$ such that $\beta < \alpha$.
Obviously, the nonpositivity of $f^{\varrho}_{ij}$ implies $W_s(\alpha) \leq W_{s-1}(\alpha^-)+\phi_s(\alpha)$. Therefore,  $$W_s(\alpha) \leq \min\left\{\min\limits_{(\alpha,\beta)\in E_{H}[\overline{\Gamma^{\cap}}]} W_s(\beta), W_{s-1}(\alpha^-)+\phi_s(\alpha)\right\}.$$ Let us prove that, in fact, the equality holds.

Let $x\in\ast\alpha$ be a word  that minimizes $W_s(\alpha)$. Then, at least one of the following two properties holds
\begin{itemize}
\item there exists $(\varrho,i,s)\in\Lambda$ such that $x_{s-\vectornorm{\varrho}+1:s}\in\varrho$ and $\varrho \not\geq \alpha$;
\item a) for any $(\varrho,i,s)\in\Lambda$, $s-i+1\leq\vectornorm{\alpha}$ such that $x_{s-\vectornorm{\varrho}+1:s}\in\varrho$ we have $\alpha_{\vectornorm{\alpha}-\vectornorm{\varrho}+1:\vectornorm{\alpha}}\subseteq\varrho$ and b) there is no $(\varrho,i,s)\in\Lambda$, $s-i+1>\vectornorm{\alpha}$ such that $x_{s-\vectornorm{\varrho}+1:s}\in\varrho$.
\end{itemize}
In the first case we obviously have $W_s(\alpha)= W_s(\alpha | \varrho)$ and $\alpha > \alpha | \varrho$. Then there exists $\beta$ such that $(\alpha,\beta)\in E_{H}[\overline{\Gamma^{\cap}}]$ and $\beta\geq \alpha | \varrho$. It is easy to see that $W_s(\alpha)\leq W_s(\beta)\leq W_s(\alpha | \varrho)$, therefore $W_s(\alpha)= W_s(\beta)$, i.e. the equality~\eqref{eq:alg} holds.

In the second case, obviously, $F_s(x) = F_{s-1}(x^-)+\phi_s(\alpha)$. Therefore, $W_s(\alpha) = W_{s-1}(\alpha^-)+\phi_s(\alpha)$, which completes the proof.
\end{proof}
From Theorem~\ref{negative} it follows that Algorithm~\ref{alg:DP'} solves the problem $PPM\left(\Gamma\right)$ with nonpositive $c^\varrho_{ij}$.% in time $O(n|E_{H}[\overline{\Gamma^{\cap}}]|)$.
\begin{algorithm}
\caption{Computing $\min_{x\in D^{n}} F(x)$}\label{alg:DP'}
\begin{algorithmic}[1]
\State initialize messages:
set $W_0(\alpha):=0$
\State Calculate $o: \overline{\Gamma^{\cap}}\rightarrow \{1,2,\cdots,|\overline{\Gamma^{\cap}}|\}$ a topological ordering of $\overline{\Gamma^{\cap}}$, i.e. such that $\forall (\alpha,\beta)\in E_{H}[\overline{\Gamma^{\cap}}]\rightarrow o(\alpha) > o(\beta)$.
\State for each $s=1,\ldots,n$ traverse nodes $\alpha\in \overline{\Gamma^{\cap}}$
in the order of increasing $o(\alpha)$ and set
\begin{equation*}
W_{s}( \alpha ) \!:=\! \min\left\{\min\limits_{(\alpha,\beta)\in E_{H}[\overline{\Gamma^{\cap}}]} W_s(\beta), W_{s-1}(\alpha^-)+\phi_s(\alpha)\right\}
\end{equation*}
If $\alpha=\varepsilon$ then use $W_{s-1}(\alpha^-)=0$
\State return
$
W_n(\varepsilon)
$
\end{algorithmic}
\end{algorithm}

\begin{remark} 
Note that Algorithm~\ref{alg:DP'} correctly solves the problem for any finite constraint language $\Gamma$, i.e. $\Gamma$ need not be simple. In the case of arbitrary weights, which is studied in the next section, a more complicated computation requires a set of messages to be indexed by elements of $\Gamma^\ast$, i.e. not by $\overline{\Gamma^{\cap}}\subseteq \Gamma^\ast$.
\end{remark}

The following theorem is straightforward.
\begin{theorem}\label{complexity1}
For nonpositive weights, the number of operations (the minimum and the summation) needed to compute $\min_{x\in D^{n}} F(x)$ is bounded by ${\mathcal O} (n |E_{H}[\overline{\Gamma^{\cap}}]|)$.
\end{theorem}
\begin{proof} The number of operations in Algorithm~\ref{alg:DP'} is bounded by ${\mathcal O} (n |E_{H}[\overline{\Gamma^{\cap}}]|)$.
In algorithm~\ref{alg:DP'} we assume that $\{\phi_s(\alpha)\}_{s\in [n], \alpha\in \overline{\Gamma^{\cap}}}$ were already computed in advance. It is easy to see that the computation of those weights needs at most ${\mathcal O} (n |E_{H}[\overline{\Gamma^{\cap}}]|)$ operations. Thus, overall we need ${\mathcal O} (n |E_{H}[\overline{\Gamma^{\cap}}]|)$ operations to compute $\min_{x\in D^{n}} F(x)$.
\end{proof}
From Theorems~\ref{complexity1} and~\ref{basicbound} we obtain a bound on the complexity of Algorithm~\ref{alg:DP'}.
\begin{theorem} For a simple language $\Gamma$ the number of operations (the minimum and the summation) used in the Algorithm~\ref{alg:DP'} is bounded by ${\mathcal O} \big(n |\overline{\Gamma^\cap}|\cdot |D|  \cdot \max_{\varrho\in \Gamma}\vectornorm{\varrho}^2\big) $. For a general language, it is bounded by ${\mathcal O} \big(n \cdot |\overline{\Gamma^\cap}|^2\big)$.
\end{theorem}
\begin{proof} Theorem~\ref{complexity1} claims that the Algorithm~\ref{alg:DP'}, together with the computation of weights $\{\phi_s(\alpha)\}_{s\in [n], \alpha\in \overline{\Gamma^{\cap}}}$ requires ${\mathcal O} (n |E_{H}[\overline{\Gamma^{\cap}}]|)$ operations. From Theorem~\ref{basicbound} we conclude that $|E_{H}[\overline{\Gamma^{\cap}}]|$ is bounded by ${\mathcal O} \big(|D| \cdot |\overline{\Gamma^\cap}| \cdot \max_{\varrho\in \Gamma}\vectornorm{\varrho}^2\big)$ from which the first statement of theorem directly follows.

For the general language we have $|E_H[\overline{\Gamma^\cap}]| \leq |\overline{\Gamma^\cap}|^2$, and the overall compexity is bounded by ${\mathcal O} \big(n \cdot |\overline{\Gamma^\cap}|^2\big)$.
\end{proof}
\section{Algorithm for $PPSS\left(\Gamma, \mathcal{R}\right)$}\label{generalAlg}
Similarly as in Section~\ref{nonpositive}, for an index $s\in[0,n]$ and a word $x$ of length $s$, we define the "partial" cost of $x$ as
\begin{equation}
F_{s}(x) = \bigotimes\limits_{\left( {\alpha ,i,j} \right)\in\Lambda, j\leq s} {f_{ij}^\alpha \left( {{x_{i:j}}} \right)}
\end{equation}

We will compute the following quantities for each $s\in[0,n]$, $\alpha\in\Gamma^*$: % called {\em messages}:
\begin{equation}
M_s(\alpha)=\!\!\!\!\!\!\bigoplus_{x\in\calX_s(\alpha;\Gamma^*)}\!\!\!\!\!\! F_s(x)
\;,\quad
W_s(\alpha)=\!\!\!\!\bigoplus_{x\in\calX_s(\alpha)}\!\!\! F_s(x)
\label{eq:DP'':MandW}
\end{equation}
where $$\calX_s(\alpha) = \left\{x\in D^s\:|\: x\in (\ast\alpha)\right\}$$ and $$\calX_s(\alpha;\Gamma^*) = \left\{x\in D^s\:|\: x\in (\ast\alpha), \forall \beta\in\Gamma^*:\alpha> \beta\rightarrow x\notin (*\beta)\right\}.$$

\begin{theorem} The following equality always holds:
\begin{equation}\label{inclusion}
W_s(\alpha) = \!\!\!\!\bigoplus_{\beta\in\Gamma^*:\alpha\geq\beta}\!\!\! M_s(\beta)
\end{equation}
\end{theorem}
\begin{proof}
The identity straightforwardly follows from the fact that
$$
\calX_s(\alpha) = \!\!\!\!\bigcup_{\beta\in\Gamma^*:\alpha\geq\beta}\!\!\! \calX_s(\beta;\Gamma^*)
$$
and the union is disjoint. 

Let us prove the disjointness. Indeed, for any two distinct $\beta_1, \beta_2\in\Gamma^*$ such that $\alpha\geq\beta_1$ and $\alpha\geq\beta_2$ we have 2 cases: a) $\beta_1 \geq \beta_2$; in that case $\calX_s(\beta_2)\subseteq \calX_s(\beta_1)$ and $\calX_s(\beta_1;\Gamma^*)\subseteq \calX_s(\beta_1)\setminus \calX_s(\beta_2)$, therefore  $\calX_s(\beta_1;\Gamma^*)\cap \calX_s(\beta_2;\Gamma^*) \subseteq  (\calX_s(\beta_1)\setminus \calX_s(\beta_2) )\cap \calX_s(\beta_2)= \emptyset$ (analogously, we deal with the case $\beta_2 \geq \beta_1$); b) $\beta_1 \not\geq \beta_2$ and $\beta_2 \not\geq \beta_1$; in that case $\calX_s(\beta_1)\cap \calX_s(\beta_2) = \calX_s(\beta_1 | \beta_2)$, $\calX_s(\beta_i;\Gamma^*)\subseteq \calX_s(\beta_i)\setminus \calX_s(\beta_1 | \beta_2)$ and, therefore, $\calX_s(\beta_1;\Gamma^*)\cap \calX_s(\beta_2;\Gamma^*) \subseteq  \calX_s(\beta_1)\cap \calX_s(\beta_2)\setminus \calX_s(\beta_1 | \beta_2) = \emptyset$.
\end{proof}
\ifCOM
\begin{proof}
For any $\alpha\in \Gamma^*$, let us define ${\rm Height}(\alpha) = \max\limits_{\alpha\geq \alpha_1 > \cdots >\alpha_l} l$ and prove~\eqref{inclusion} by induction on ${\rm Height}(\alpha)$. 

The basis of induction: for ${\rm Height}(\beta)=1$, $\beta$ is a minimal element in the partial order $(\Gamma^*, \geq)$ and $W_s(\beta) = M_s(\beta)$. 

The induction step: let us assume that~\eqref{inclusion} holds for all $\beta\in \Gamma^*$ such that ${\rm Height}(\beta)\leq m$. Suppose we are given $\alpha\in \Gamma^*$ and ${\rm Height}(\alpha)= m+1$. By inclusion-exclusion principle we have:
$$
(*\alpha) = (*\alpha)
$$
\end{proof}
\else
\fi

\begin{remark}
Equality~\eqref{inclusion} can be used for computing $W_s(\cdot)$ based on the already known values  $M_s(\cdot)$. If the summation in the semiring is idempotent, i.e. $x\oplus x=x$, then instead of summing over all 
elements non-greater than $\alpha$, we can perform the computation starting from the minimal elements of $\left(\Gamma^*,\geq\right)$ and upwards, according to:
\begin{equation}
W_s(\alpha) = M_s(\alpha)\oplus\!\!\!\!\bigoplus_{(\alpha,\beta)\in E_{H}[\Gamma^\ast]}\!\!\! W_s(\beta)
\end{equation}
This allows a faster computation running in time $O\left(|E_{H}[\Gamma^\ast]|\right)$ for $\mathcal{R}=(\mathbb Q,\min,+)$.
\end{remark}

It is easy to see that Algorithm~\ref{alg:DP2} computes $\oplus_{x\in D^{n}} F(x)$ if and only if Theorem~\ref{formula} holds.

\begin{algorithm}
\caption{Computing $\oplus_{x\in D^{n}} F(x)$}\label{alg:DP2}
\begin{algorithmic}[1]
\State Initialize messages:
set $M_0(\alpha):=\mathbb{O}$ for $\alpha \in \Gamma^\ast$
\State Calculate $\Gamma^\ast_0 = \{\beta\in \Gamma^\ast |\,\, |{\rm proj}_{\vectornorm{\beta}}\beta| = 1\}$
\State For each $s=1,\ldots,n$ and $\alpha \in \Gamma^\ast\setminus \Gamma^\ast_0 $, set $M_{s}( \alpha ) \!:=\!\mathbb{O}$.
\State Calculate $o: \Gamma^\ast_0\rightarrow [|\Gamma^\ast_0|]$ a topological ordering of $\Gamma^\ast_0$, i.e. such that $\forall (\alpha,\beta)\in E_{H}[\Gamma^\ast_0]\rightarrow o(\alpha) > o(\beta)$.
\State For each $s=1,\ldots,n$ traverse nodes $\alpha\in \Gamma^\ast_0$
in the order of increasing $o(\alpha)$ and set
\begin{equation*}
\begin{split}
\phi_s(\alpha)\!:=\! \mathop\otimes\limits_{\beta\in\Gamma:\beta\geq\alpha}c^{\beta}_{s-\vectornorm{\beta}+1,s} \\
M_{s}( \alpha ) \!:=\! \phi_s(\alpha)\otimes \bigoplus_{\beta\in\Gamma^*_0:\alpha^-\geq\beta, \forall (\alpha,\gamma)\in E_{H}[\Gamma^\ast]\,\,\gamma^-\not\geq\beta}\!\!\! M_{s-1}(\beta)
\end{split}
\end{equation*}
If $\alpha=\varepsilon$ then use $M_{s-1}(\alpha^-)=\mathbb{O}$
\State return
$
\oplus_{\alpha\in\Gamma^\ast}M_n(\alpha)
$
\end{algorithmic}
\end{algorithm}

\begin{theorem}\label{formula} Let $\alpha\in\Gamma^*$ satisfy $\vectornorm{\alpha}\geq 1$. Then, for $|{\rm proj}_{\vectornorm{\alpha}}\alpha|\ne 1$ we have $M_s(\alpha) = \mathbb{O}$. If $|{\rm proj}_{\vectornorm{\alpha}}\alpha| = 1$, then 
\begin{equation}
\label{eq:alg2}
M_s(\alpha) = \phi_s(\alpha)\otimes \bigoplus_{\beta\in\Gamma^*_0:\alpha^-\geq\beta, \forall (\alpha,\gamma)\in E_{H}[\Gamma^\ast]\,\,\gamma^-\not\geq\beta}\!\!\! M_{s-1}(\beta)
\end{equation}
where $\phi_s(\alpha)\triangleq\mathop\otimes\limits_{\beta\in\Gamma:\beta\geq\alpha}c^{\beta}_{s-\vectornorm{\beta}+1,s}$ and $\Gamma^\ast_0 = \{\beta\in \Gamma^\ast |\,\, |{\rm proj}_{\vectornorm{\beta}}\beta| = 1\}$.
\end{theorem}
To prove Theorem~\ref{formula} we need the following lemma.
\begin{lemma}\label{factor} Let $\alpha\in\Gamma^*$ satisfy $\vectornorm{\alpha}\geq 1$. 
Then, for $|{\rm proj}_{\vectornorm{\alpha}}\alpha|\ne 1$ we have $\calX_s(\alpha;\Gamma^*) = \emptyset$.
%then $\calX_s(\alpha;\Gamma^*) \subseteq T\times \Omega$
If $|{\rm proj}_{\vectornorm{\alpha}}\alpha| = 1$, then $\calX_s(\alpha;\Gamma^*) = T\times {\rm proj}_{\vectornorm{\alpha}}\alpha$ where $T = \{x:|x|=s-1, x\in\ast\alpha^-,\forall (\alpha,\gamma)\in E_{H}[\Gamma^\ast]\,\,x\notin\ast(\gamma^-)\}$. %and  $\Omega = \{a\in D| \exists w\in D^{s-1}, wa\in\calX_s(\alpha;\Gamma^*)\}$.
\end{lemma}
\begin{proof}
First, let us show that for $|{\rm proj}_{\vectornorm{\alpha}}\alpha|\ne 1$ we have $\calX_s(\alpha;\Gamma^*) = \emptyset$. For all $a\in {\rm proj}_{\vectornorm{\alpha}}\alpha$ we have $\alpha^- \times \{a\}\in \Gamma^*$, therefore $\beta_a =\alpha | (\alpha^- \times \{a\})\in \Gamma^*$. It is easy to see that $\beta_a < \alpha$ for any $a\in {\rm proj}_{\vectornorm{\alpha}}\alpha$. Consequently, $\calX_s(\alpha;\Gamma^*) \subseteq \calX_s(\alpha)\setminus \mathop{\cup}\limits_{a\in {\rm proj}_{\vectornorm{\alpha}}\alpha}\calX_s(\beta_a) = \emptyset$.

Let us now assume that $|{\rm proj}_{\vectornorm{\alpha}}\alpha| = 1$, $\alpha = \alpha^{-1} \times {\rm proj}_{\vectornorm{\alpha}}\alpha$ and show the next statement of Lemma.
Let $x= x_{1:s}\in \calX_s(\alpha;\Gamma^*)$, i.e. $x\in D^s$ and the following is true: a) $x\in (\ast \alpha)$, b) $\forall \beta\in\Gamma^*:\alpha> \beta\rightarrow x\notin (*\beta)$. From the property a) it is obvious that $x^{-}\in (\ast \alpha^{-})$. The situation with the property b) is more complicated, i.e. for any $\beta\in\Gamma^*$ such that $\alpha> \beta$ we have two subclasses: 1) $x^{-}\in (*\beta^{-})$, and  2) $x^{-}\notin (*\beta^{-})$.

Let us now consider the first case.
Since $\Gamma^\ast$ is singleton suffix closed, then we have $\beta^{-}\times \{x_{s}\}\in \Gamma^*$ and $x\in (\ast (\beta^{-}\times \{x_{s}\}))$, and consequently $x\in (\ast (\alpha | (\beta^{-}\times \{x_{s}\})))$. From the definition of $\calX_s(\alpha;\Gamma^*)$ and $x\in (\ast (\alpha | (\beta^{-}\times \{x_{s}\})))$ we conclude $\alpha | (\beta^{-}\times \{x_{s}\}) = \alpha$ and, consequently, $\alpha \leq \beta^{-}\times \{x_{s}\}$. Therefore, $\beta<\alpha \leq \beta^{-}\times \{x_{s}\}$ and $\beta\subset \beta^{-}\times \{x_{s}\}$. It is easy to see that  the last containment is impossible. Thus, we proved that only the case 2 is possible, which means $\calX_s(\alpha;\Gamma^*) \subseteq T\times {\rm proj}_{\vectornorm{\alpha}}\alpha$.

Let us now prove the inverse inclusion, i.e. $T\times {\rm proj}_{\vectornorm{\alpha}}\alpha \subseteq \calX_s(\alpha;\Gamma^*)$. Let $x\in T$ and $a\in {\rm proj}_{\vectornorm{\alpha}}\alpha$, i.e. $x\in D^{s-1}$ and $x\in\ast\alpha^-$, $\forall \gamma\in\Gamma^\ast: \gamma < \alpha\rightarrow x\notin (\ast\gamma^-)$. Then, $xa\notin (\ast\gamma)$, because $x\notin (\ast\gamma^-)$. Since $\alpha = \alpha^{-1} \times \{a\}$, then $xa\in \ast\alpha$. Thus, $xa\in \calX_s(\alpha;\Gamma^*)$. 
\end{proof}
\begin{proof}[Proof of Theorem~\ref{formula}]
If $|{\rm proj}_{\vectornorm{\alpha}}\alpha|\ne 1$, then Lemma gives $\calX_s(\alpha;\Gamma^*) = \emptyset$, and therefore $M_s(\alpha) = \mathbb{O}$. Thus, only messages for patterns from $\Gamma^*_0$ can be nonzero.

Let $\alpha\in \Gamma^*_0$.
The definition of $\calX_s(\alpha;\Gamma^*)$ implies that for any $x\in \calX_s(\alpha;\Gamma^*)$, $F_s (x) = \phi_s(\alpha)\otimes F_{s-1}(x^-)$. Indeed, for any $x\in \calX_s(\alpha;\Gamma^*)$, the term $c^{\beta}_{s-\vectornorm{\beta}+1,s}$ contributes to $F_s (x)$ if and only if $x\in (\ast \beta)$. Since $x\in (\ast \alpha)$, then $x\in (\ast (\alpha|\beta))$. If $\alpha|\beta < \alpha$ we have a contradiction with $\calX_s(\alpha;\Gamma^*)\subseteq \calX_s(\alpha)\setminus \calX_s(\alpha|\beta)$. Therefore, $\alpha|\beta = \alpha$ and $\beta \geq \alpha$. Thus, from a set of boundary terms only terms $c^{\beta}_{s-\vectornorm{\beta}+1,s}$, $\beta\in\Gamma:\beta\geq\alpha$  contribute to $F_s (x)$ and they are all present in $\phi_s(\alpha)$.

Using Lemma~\ref{factor}, we obtain
\[
M_s(\alpha) = \phi_s(\alpha)\otimes \bigoplus_{x\in T} F_{s-1}(x)
\]
where $T = \left\{x:|x|=s-1, x\in\ast\alpha^-,\forall (\alpha,\gamma)\in E_{H}[\Gamma^\ast]\,\,x\notin\ast(\gamma^-)\right\}$. 

We will prove the statement of the theorem using the following representation of $T$ as a disjoint union of sets:
\[
T = \bigcup_{\beta\in\Gamma^*:\alpha^-\geq\beta, \forall (\alpha,\gamma)\in E_{H}[\Gamma^\ast]\,\,\gamma^-\not\geq\beta} \calX_{s-1}(\beta;\Gamma^*)
\]
Let us prove this representation.

For any $x\in D^{s-1}$ there is $\beta_{x}\in \Gamma^*$ such that 
\begin{equation}\label{put-it}
\calX_{s-1}(\beta_{x}) =   \bigcap\limits_{\beta\in\Gamma^*:x\in\calX_{s-1}(\beta)} \calX_{s-1}(\beta).
\end{equation}
Since $\calX_{s-1}(\beta_1\mid\beta_2) = \calX_{s-1}(\beta_1)\cap \calX_{s-1}(\beta_2)$ we can define $\beta_x = \beta_1 | \cdots | \beta_N$ where $\{\beta\in \Gamma^\ast: x\in (\ast\beta)\} = \{\beta_i \mid i\in [N]\}$. From the construction of $\beta_x$, for any $x\in D^{s-1}$ we have $x\in \calX_{s-1}(\beta_{x};\Gamma^*)$.

Let $x\in T$, i.e. $x\in (\ast\alpha^-)$ and $\forall\,\,\gamma\in \Gamma^\ast: \alpha>\gamma \rightarrow x\notin (\ast\gamma^-)$. From the latter we obtain $\beta_x \leq \alpha^-$ and $\beta_x \not\leq \gamma^-$. Thus, $\calX_{s-1}(\beta_x;\Gamma^*)$ will be present in the right hand side of the expression~\eqref{put-it}. In other words, the inclusion of $T$ in $\bigcup_{\beta\in\Gamma^*:\alpha^-\geq\beta, \forall (\alpha,\gamma)\in E_{H}[\Gamma^\ast]\,\,\gamma^-\not\geq\beta} \calX_{s-1}(\beta;\Gamma^*)$ is proved.

Suppose now  $x\in \calX_{s-1}(\alpha^-)\setminus T$. Then, there is $(\alpha,\gamma)\in E_{H}[\Gamma^\ast]$ such that $x\in\ast(\gamma^-)$. Then $\gamma^-\geq\beta_{x}$ and 
$\calX_{s-1}(\beta_x;\Gamma^*)$ is not present in the right hand side expression. Since $\calX_{s-1}(\beta;\Gamma^*)$ are all nonintersecting for different $\beta$ and  $x\in \calX_{s-1}(\beta_{x};\Gamma^*)$, then 
$x$ is not in any set of the right hand side union. 

Finally, if $\beta\in\Gamma^*:\alpha^-\geq\beta, \forall (\alpha,\gamma)\in E_{H}[\Gamma^\ast], \gamma^-\not\geq\beta$ is such that $\beta\notin\Gamma^*_0$, then $\calX_{s-1}(\beta;\Gamma^*) = \emptyset$. Therefore, $T = \bigcup_{\beta\in\Gamma^*_0:\alpha^-\geq\beta, \forall (\alpha,\gamma)\in E_{H}[\Gamma^\ast]\,\,\gamma^-\not\geq\beta} \calX_{s-1}(\beta;\Gamma^*)$ and from that the equation~\eqref{eq:alg2} directly follows.
This completes the proof.
\end{proof}
%The following theorem is straightforward.
\begin{theorem} \label{complexity2}
The number of operations ($\oplus$ and $\otimes$) used in Algorithm~\ref{alg:DP2} is bounded by ${\mathcal O} (n |\Gamma^\ast|^2)$. For a simple language $\Gamma$, it is bounded by ${\mathcal O} (n |D|^2 |\overline{\Gamma^{\cap}}|^2)$.
\end{theorem}
\begin{proof} The first statement is obvious. The second statement follows from Inequality~\eqref{nextsimple}.
\end{proof}

\section{A bound on $|E_H[\overline{\Gamma^\cap}]|$}\label{bounds}
Let us denote $\alpha\succ \beta$ if and only if $\vectornorm{\beta}>\vectornorm{\alpha}, \beta_{\vectornorm{\beta}-\vectornorm{\alpha}+1:\vectornorm{\beta}}\subseteq\alpha$. Also, $\alpha\succeq \beta$ if $\alpha\succ \beta$ or $\alpha=\beta$.
It easy to see that $\geq = \succeq\cup \supseteq$, where $\supseteq$ is the  partial order on $\overline{\Gamma^\cap}$ defined by the standard inclusion of predicates. Moreover, $\succ$ and $\supset$ are nonintersecting.

Let us introduce two  graphs $G[\overline{\Gamma^\cap}]$ and $H[\overline{\Gamma^\cap}, \supseteq]$. For both of them, the  set of vertices is equal to $\overline{\Gamma^\cap}$. The set of edges for $G[\overline{\Gamma^\cap}]$, denoted by $E_{G}[\overline{\Gamma^\cap}]$, is equal to
$$E_{G}[\overline{\Gamma^\cap}] = \left\{(\alpha,\beta)|\alpha\succ\beta, \mbox{there\,\,is\,\,no\,\,} \gamma\in \overline{\Gamma^\cap}\mbox{\,\,s.t.\,\,}\alpha\succ\gamma\succ\beta\mbox{\,\,or\,\,}\alpha\supset\gamma\succ\beta\right\}.$$
%Equivalently, $(\alpha,\beta)\in E_{G}[\Gamma]$ if and only if $(\alpha,\beta)$ belongs to Hasse diagram of $\succ$ and $\beta_{\vectornorm{\beta}-\vectornorm{\alpha}+1:\vectornorm{\beta}}=\alpha$. Indeed, the fact that $E_{G}[\Gamma]$ is a part of Hasse diagram of $\prec$ is obvious. Suppose now that $\beta_{\vectornorm{\beta}-\vectornorm{\alpha}+1:\vectornorm{\beta}}\subset\alpha$.

An edge set for $H[\overline{\Gamma^\cap}, \supseteq]$, denoted $E_{H}[\overline{\Gamma^\cap}, \supseteq]$, is equal to $\big\{(\alpha,\beta)|\alpha\supset\beta, \forall \gamma \mbox{ s.t. }\alpha\supset\gamma\supset\beta: \gamma\notin\overline{\Gamma^\cap}\big\}$.
Obviously, $E_{H}[\overline{\Gamma^\cap}, \supseteq]$ is just the Hasse diagram for $\supseteq$ on $\overline{\Gamma^\cap}$.

First we need to bound $\left|E_{G}[\overline{\Gamma^\cap}]\right|$ and $\left|E_{H}[\overline{\Gamma^\cap}, \supseteq]\right|$. For this purpose we will need the following lemma:

\begin{lemma}\label{durak} Let $\Gamma$ be simple and $\varrho$ be any predicate over $D$, and $\Gamma(\varrho) = \{\alpha\in \overline{\Gamma^\cap} | \alpha\supset \varrho, \forall \gamma \mbox{ s.t. }\alpha\supset\gamma\supset\varrho: \gamma\notin\overline{\Gamma^\cap}\}$.
Then we have $\left|\Gamma(\varrho)\right|\leq \vectornorm{\varrho} \cdot |D|$.
\end{lemma}
\begin{proof}
Since any predicate $\alpha\in \overline{\Gamma^\cap}$ is equal to some direct product $\Omega_1\times\cdots\times\Omega_{\vectornorm{\alpha}}$, we can correspond to $\alpha$ a subset $c(\alpha) = \left\{(a,i)|a\in\Omega_i\right\}\subseteq  D \times \left\{1,\cdots,\vectornorm{\alpha}\right\}$. Then $\alpha\subset \beta$ (or, $\alpha = \beta$) if and only if $c(\alpha)\subset c(\beta)$ ($c(\alpha)= c(\beta)$) and if $\alpha\cap\beta\ne \emptyset$, then $c(\alpha\cap\beta)=c(\alpha)\cap c(\beta)$.

If $|\Gamma(\varrho)|=1$, then statement of lemma is obvious. Suppose $|\Gamma(\varrho)|>1$.

If $\alpha,\beta\in \Gamma(\varrho)$ and $\alpha\ne \beta$, then $\varrho\subseteq\alpha\cap\beta\in \overline{\Gamma^\cap}$. Moreover, $\alpha\cap\beta\ne \alpha$, because otherwise we have $\beta\supset \alpha\cap\beta = \alpha\supset\varrho$ and we get a contradiction with $\beta\in \Gamma(\varrho)$. Analogously, $\alpha\cap\beta\ne\beta$.
But since $\alpha\cap\beta\in \overline{\Gamma^\cap}$ and $\varrho\subseteq\alpha\cap\beta\subset \alpha$, then $\varrho = \alpha\cap\beta$. Thus, $\varrho$ is itself a cartesian product and $c(\varrho) = c(\alpha)\cap c(\beta)$. Therefore, $(c(\alpha)\setminus c(\varrho))\cap (c(\beta)\setminus c(\varrho)) = \emptyset$.

From the latter identity we obtain that $\cup_{\alpha:\alpha\in \Gamma(\varrho)} (c(\alpha)\setminus c(\varrho))$ is a union of nonempty and disjoint sets that is a subset of $ D \times \left\{1,\cdots,\vectornorm{\varrho}\right\}$. Therefore, $\left|\Gamma(\varrho)\right|\leq | D \times \left\{1,\cdots,\vectornorm{\varrho}\right\}| = |D| \cdot \vectornorm{\varrho}$.
\end{proof}

Let us now consider $G[\overline{\Gamma^\cap}]$. 
\begin{lemma} If $\Gamma$ is simple, then
$\left|E_{G}[\overline{\Gamma^\cap}]\right|\leq \sum_{\alpha\in \overline{\Gamma^\cap}}(\vectornorm{\alpha}^2-\vectornorm{\alpha}) \cdot |D|/2$.
\end{lemma}
\begin{proof}
By construction, for fixed $\beta$, $\left\{\alpha|(\alpha,\beta)\in E_{G}[\overline{\Gamma^\cap}]\right\} \subseteq \bigcup_{i=2}^{\vectornorm{\beta}} \Gamma (\beta_{i:\vectornorm{\beta}})$. Therefore, using Lemma~\ref{durak}, we have:
\begin{equation*}
\begin{split}
\left|E_{G}[\overline{\Gamma^\cap}]\right|\leq \sum_{\beta\in \overline{\Gamma^\cap}}\sum_{i=2}^{\vectornorm{\beta}} |\Gamma (\beta_{i:\vectornorm{\beta}})| \leq \sum_{\beta\in \overline{\Gamma^\cap}} \sum_{i=2}^{\vectornorm{\beta}} (\vectornorm{\beta}-i+1) |D| = 
\sum_{\beta\in \overline{\Gamma^\cap}}\frac{\vectornorm{\beta}^2-\vectornorm{\beta}}{2} \cdot |D|.
\end{split}
\end{equation*}  
\end{proof}
Let us now consider $H[\overline{\Gamma^\cap}, \supseteq]$. For it the following lemma holds:
\begin{lemma} If $\Gamma$ is simple, then
$\left|E_{H}[\overline{\Gamma^\cap}, \supseteq]\right|\leq \sum_{\alpha\in \overline{\Gamma^\cap}}\vectornorm{\alpha} \cdot |D|$.
\end{lemma}
\begin{proof}
Since $E_{H}[\overline{\Gamma^\cap}, \supseteq]\subseteq \bigcup_{\alpha\in \overline{\Gamma^\cap}}  \Gamma (\alpha) \times\{\alpha\}$, then using Lemma~\ref{durak} we obtain:
\begin{equation*}
\left|E_{H}[\overline{\Gamma^\cap}]\right|\leq \sum_{\alpha\in \overline{\Gamma^\cap}}\vectornorm{\alpha} \cdot |D|.
\end{equation*}  
\end{proof}

\begin{lemma}\label{transitive} If $\alpha \succ \beta$, $\alpha,\beta\in\overline{\Gamma^\cap}$, then there is $\gamma\in\overline{\Gamma^\cap}, \gamma\ne \alpha$ such that $(\alpha,\gamma)\in E_{G}[\overline{\Gamma^\cap}]\cup E_{H}[\overline{\Gamma^\cap}, \supseteq]$ and $\gamma \succeq \beta $.
\end{lemma}
\begin{proof}
Consider a predicate $\kappa = \beta_{\vectornorm{\beta}-\vectornorm{\alpha}+1:\vectornorm{\beta}}$ and $\theta = \cap_{\eta:\kappa\subseteq\eta\in\overline{\Gamma^\cap}}\eta$. Obviously, $\theta\in \overline{\Gamma^\cap}$ and $\theta\subseteq\alpha$.
If $\theta\subset\alpha$, then there is $\gamma\ne \alpha$, $(\alpha,\gamma)\in E_{H}[\overline{\Gamma^\cap}, \supseteq]$ such that $\theta\subseteq\gamma$, which implies $  \gamma \succ \beta$.

Suppose now that $\theta=\alpha$. Then there is $\gamma\in\overline{\Gamma^\cap}$ such that $\alpha\succ\gamma\succeq \beta$ and there is no $\gamma'\in\overline{\Gamma^\cap}$ such that $\alpha\succ\gamma'\succ\gamma$ (i.e. $(\alpha, \gamma)$ is in the Hasse diagram for $\succ$). Since any $\eta\in\overline{\Gamma^\cap}$ such that $\eta\supseteq\kappa$ also contains $\alpha$, then there is no $\gamma''\in\overline{\Gamma^\cap}:\alpha\supset\gamma''\succ\gamma $. Indeed, if such $\gamma''$ would exist, then $\gamma''\succeq \beta\Rightarrow \gamma''\supseteq \kappa$, and therefore, after setting $\eta = \gamma''$, we obtain $\alpha\subseteq \gamma''$. The latter makes a contradiction. Therefore, $(\alpha,\gamma)\in E_{G}[\overline{\Gamma^\cap}]$.
\end{proof}

\begin{proof}[Proof of Theorem~\ref{basicbound}]
It is easy to see that the transitive closure of $E_{G}[\overline{\Gamma^\cap}]\cup E_{H}[\overline{\Gamma^\cap}, \supseteq]$ equals $\geq$. Indeed, if $\alpha\geq \beta$ then we have two cases: 1) $\alpha \supseteq \beta$, 2) $\alpha \succ \beta$. If $\alpha \supseteq \beta$, then $\beta$ is reachable from $\alpha$ using edges of $E_{H}[\overline{\Gamma^\cap}, \supseteq]$. If we have $\alpha \succ \beta$, then from Lemma~\ref{transitive} it is clear that we can reach first $\gamma\in\overline{\Gamma^\cap}, \gamma\ne \alpha$ through the edge $(\alpha,\gamma)\in E_{G}[\overline{\Gamma^\cap}]\cup E_{H}[\overline{\Gamma^\cap}, \supseteq]$ and then reach $\beta$.

The Hasse diagram is a smallest digraph whose transitive closure equals $\geq$, therefore
\begin{equation*}
\begin{split}
| E_{H}[\overline{\Gamma^\cap}] | \leq | E_{G}[\overline{\Gamma^\cap}]\cup E_{H}[\overline{\Gamma^\cap}, \supseteq] | \leq \sum_{\alpha\in \overline{\Gamma^\cap}}(\vectornorm{\alpha}^2-\vectornorm{\alpha}) \cdot |D|/2+
\sum_{\alpha\in \overline{\Gamma^\cap}}\vectornorm{\alpha} \cdot |D|
= \\
\frac{1}{2}\sum_{\alpha\in \overline{\Gamma^\cap}}(\vectornorm{\alpha} ^2+\vectornorm{\alpha} )\cdot |D|.  
\end{split}
\end{equation*}
\end{proof}

\ifCOM
Now the main problem is the following: how to split each $\varrho\in \Gamma$ on $s(\varrho)$ parts, i.e. $\varrho = \cup_{i=1}^{s(\varrho)}\varrho^{i}$ (such that each $\varrho^{i}$ is a cartesian product), such that $\left|\left\{\varrho^i|\varrho\in\Gamma,i=1,\cdots,s(\varrho)\right\}^*\right|$ is minimized.
I believe that the program maximum is algorithm with complexity $O\left(n f(\sum_{\varrho\in \Gamma}\vectornorm{\varrho})\right)$, where $f$
is subexponential.
\else
\fi

\section{Examples of moderate size of $\overline{\Gamma^{\cap}}$}\label{Examples}

\begin{example}\label{simplest} Suppose $\sim$ is an equivalence relation on $D$, $D/\sim$ is a set of equivalence classes with respect to $\sim$, and $\Pi$ is a prefix-closed set of strings over $D/\hspace{-2pt}\sim$, i.e. for any $\alpha\in \Pi$, $\alpha^-\in \Pi$. Let us denote
\begin{equation*}
\Gamma = \{ \Omega_1\times \cdots \times \Omega_l \mid \Omega_1 \cdots  \Omega_l\in \Pi\}
\end{equation*}
We have $\overline{\Gamma^{\cap}} = \Gamma$. Note that, in general, $\sum_{\varrho\in\Gamma}|\varrho|$ can depend exponentially on $|D|$.

There are many problems to which this example can be potentially applied such as, for example, language modeling or protein dihedral angles prediction~\cite{zhalgas}. Let us briefly outline such an application. Suppose that our goal is to build an $n$-gram model for a certain natural language $\mathcal{L}$. Then, $D$ is a set of possible words in $\mathcal{L}$ and can be as large as $10^4-10^6$. Since, the number of possible 4-grams, 5-grams etc. grows very fast, a standard $n$-gram model, for $n>3$, will be prone to overfitting or even impossible to build. Then, one can define a mapping $r: D\to D'$, where $|D'|\ll |D|$. E.g. this can be a function that maps a word into its root\footnote{In some languages, such as agglutinative languages, the number of roots is substantially smaller than the number of possible words (due to a composite structure of words).}, or a function that maps a word $w$ to its cluster in some word representation space.  Let us define an equivalence relation by $w_1\sim w_2\Leftrightarrow r(w_1)=r(w_2)$. Note that for surjective $r$, we have $(D/\sim)\equiv D'$, and we can identify corresponding elements of $(D/\sim)$ and $D'$. Then, any prefix-closed set of string $\Pi$ over $D/\sim$ is equivalent to a set sequences over $D'$. 
By construction, $\Gamma^\ast$ is prefix closed, intersection closed, singleton suffix closed, and we have $|\Gamma^\ast| \leq (|D|+1)|\Gamma|=(|D|+1)|\Pi|$, $\Gamma^\ast\supseteq \{\{a\}\mid a\in D\}$. Since $(\Gamma\cup \{\{a\}\mid a\in D\})^\ast = \Gamma^\ast$, the Algorithm~\ref{alg:DP2} allows us to compute a partition function and MAP for the following probabilistic model
\begin{equation*}
P(x_1,\cdots,x_n)\propto \exp\big(\sum_{i=1}^n\sum_{a\in D} w_{ia}[x_i=a]+\sum_{\varrho\in \Gamma}\sum_{i,j\in [n]: j-i+1=\|\varrho\|} c^\varrho_{ij}[x_{i:j}\in \varrho]\big)
\end{equation*}
or, an equivalent model
\begin{equation*}
P(x_1,\cdots,x_n)\propto \exp\big(\sum_{i=1}^n\sum_{a\in D} w_{ia}[x_i=a]+\sum_{\alpha\in \Pi}\sum_{i,j\in [n]: j-i+1=|\alpha|} c^\alpha_{ij}[r(x_{i:j})=\alpha]\big)
\end{equation*}
in time ${\mathcal O} (n (|D|+1)^2 |\Pi|^2)$ (here $r(x_{i:j}) = r(x_i) \cdots r(x_j)$).
This model is a pattern-based model over roots that is enhanced by inclusion of unary terms depending on an assignment of each $x_i$ to a specific word.
\end{example}

\begin{example} Let $\mathcal{F}\subseteq 2^D$ be an intersection-closed family of subsets of $D$ and $r_1, r_2\in {\mathbb N}$. Let us denote
\begin{equation*}
\begin{split}
\Gamma = \{ \Omega_1\times \cdots \times \Omega_l\times \{a_1\}\times \cdots \times \{a_k\} \mid \Omega_i \in \mathcal{F}, i\in [l], a_j\in D,  j\in [k], \\ l\leq r_1, k\leq r_2\}
\end{split}
\end{equation*}
It is easy to see that $\overline{\Gamma^{\cap}} = \Gamma$ (and, even $\Gamma^\ast = \Gamma$). Note that $|\Gamma| = {\mathcal O}(|\mathcal{F}|^{r_1}|D|^{r_2})$, which is favourable to $\sum_{\varrho\in\Gamma}|\varrho|> |D|^{r_1+r_2}$ if $|\mathcal{F}| < |D|$. 

If an order of letters in a sequence is reversed, then this example corresponds to a constraint language $\Gamma'$ with relations of type $\{a_1\}\times \cdots \times \{a_k\}\times \Omega_1\times \cdots \times \Omega_l$ where $\Omega_i \in \mathcal{F}, i\in [l], a_j\in D,  j\in [k], l\leq r_1, k\leq r_2$. The language $\Gamma'$ can be usefull in a statistical modeling of a stochastic process $\{X_t\}_{t\in {\mathbb Z}}, X_t\in [a,b]$, where two discretizations of $[a,b]$ are used. A finer one is used for prediction of $X_t$ in the near future, e.g. with a step size $h$ and, therefore, we set $D=\{ih\mid i\in {\mathbb Z}, i h\in [a,b]\}$, and a coarse one is used for prediction of $X_t$ in a longer future. E.g., if the step size of the coarse discretization equals some multiple of $h$, i.e. $H=Mh, M\in {\mathbb N}$, then one can define $\mathcal{F} = \{\{Hi, Hi+h, \cdots, Hi+h(M-1)\}\cap [a,b] \mid i\in {\mathbb Z}\}$, i.e. $\mathcal{F}$ is defined as a set of equivalence classes that correspond to coarse values of the process. 

Since $\Gamma^\ast = \Gamma$, the Algorithm~\ref{alg:DP2} computes (if applied in a reversed order of variables) a partition function and MAP for the following probabilistic model
\begin{equation*}
P(x_1,\cdots,x_N)\propto \exp\big(\sum_{\varrho\in \Gamma'}\sum_{i,j\in [N]: j-i+1=\|\varrho\|} c^\varrho_{ij}[x_{i:j}\in \varrho]\big).
\end{equation*}
in time  ${\mathcal O} (n |\Gamma|^2)$.
In this generalized pattern-based model, the constraint $(X_{t_0},\cdots, X_{t_0+k+l-1})\in \{a_1\}\times \cdots \times \{a_k\}\times \Omega_1\times \cdots \times \Omega_l$ is equivalent to fixing a trajectory of $X_t$ in the beginning and giving a possible interval of variation starting from $t=t_0+k$ (see Figure~\ref{3trends}). 
\end{example}

\begin{figure}

  \begin{center}
\includegraphics[width=6cm]{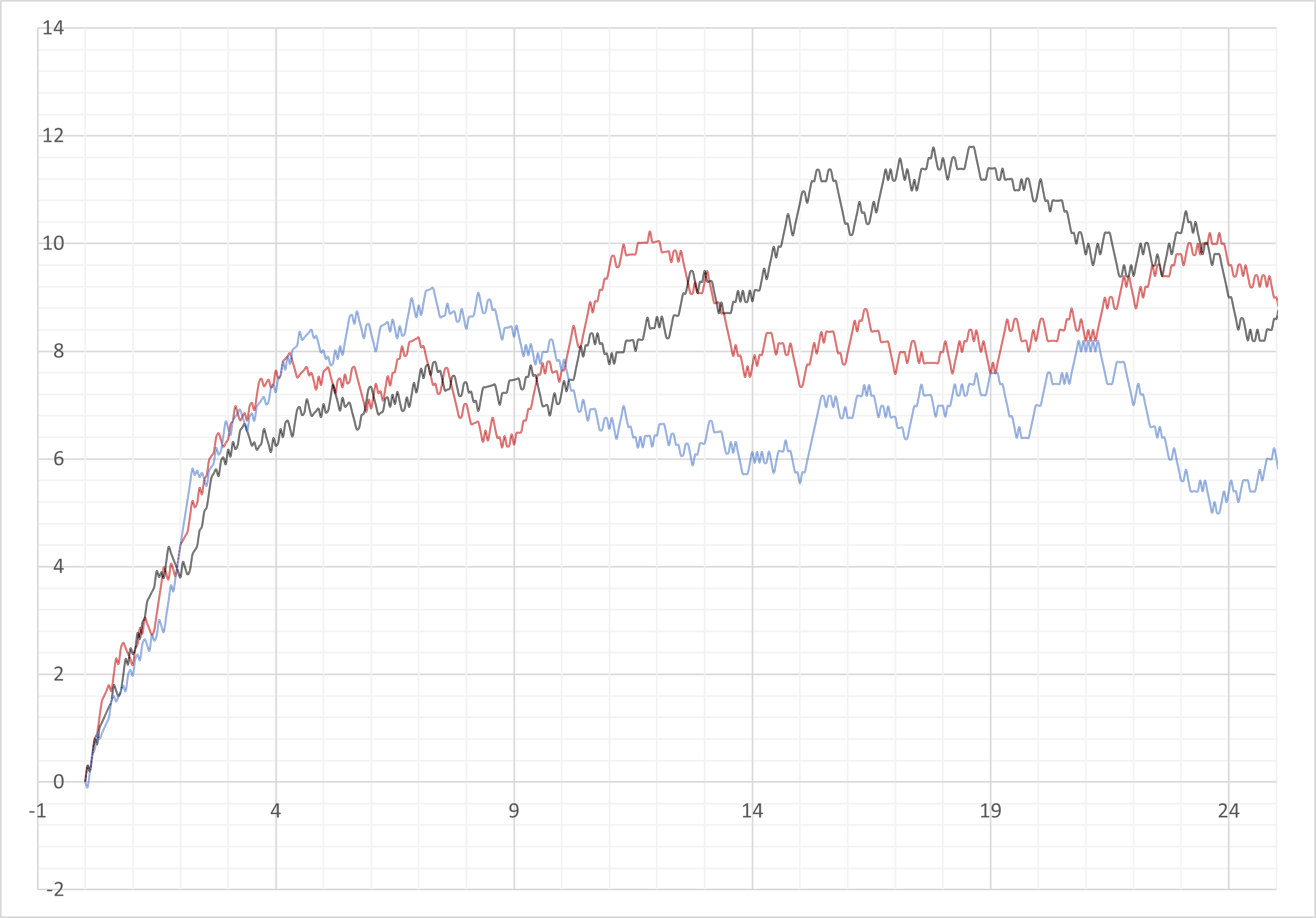}
  \end{center}
\caption{ This figure shows a typical growth behaviour on financial markets. If we choose a discretization of  both axes (time and value) with a unit step, then for first 11 moments of time we have $(X_0,X_1,\cdots,X_{10})\in \{0\}\times \{2\}\times \{4\}\times \{6\}\times \Omega^{7}$ where $\Omega = \{6,7,8,9\}$ (for all three graphs). In other words, $X_t$ is precisely known in the beginning of the process and, further, one can predict it to be in a certain interval of values.}\label{3trends}
\end{figure}

Tightly connected with previous examples is the following one.
\begin{example} Let ${\mathcal L}$ be a set of equivalence relations $\{\gamma^1, \cdots, \gamma^N\}$ over the same set $D$ such that $\gamma^1\supset \cdots \supset \gamma^N$. A set of equivalence classes of the relation $\gamma^i$ is denoted by $\mathcal{F}_i$, for $i\in [N]$. Let $\Pi$ be a prefix-closed set of strings over $[N]$ such that for any $u=u_1u_2\cdots u_n \in \Pi$ and $v=v_1v_2\cdots v_m \in \Pi$, natural numbers $n\geq m$, we have $$u_{1:n-m}(\max\{u_{n-m+1}, v_1\} \max\{u_{n-m+2}, v_2\}\cdots \max\{u_{n}, v_m\}) \in \Pi.$$ Then,
\begin{equation*}
\Gamma = \{ \Omega_1\times \cdots \times \Omega_l| \Omega_i \in \mathcal{F}_{n_i}, i\in [l], w = {n_1} {n_2} \cdots  {n_l}\in \Pi\}
\end{equation*}
satisfies $\overline{\Gamma^{\cap}} = \Gamma$. If $N=1$ and $\Pi=\{\lambda, 1, 11, 111, \cdots\}$ we obtain the example~\ref{simplest}.

For example, if $\Pi$ is the set of all nondecreasing sequences of $[N]$ with length $\leq r$ we obtain the following special case: 
\begin{equation*}
\Gamma = \{ \Omega_1\times \cdots \times \Omega_l| \Omega_i \in \mathcal{F}_{n_i}, i\in [l], {n_1}\leq {n_2} \leq \cdots \leq {n_l}, l\leq r\}
\end{equation*}
It is easy to see that $\overline{\Gamma^{\cap}} = \Gamma$. Then, $|\Gamma| = {\mathcal O}(|\mathcal{F}_N|^k)$, which is favourable to $\sum_{\varrho\in\Gamma}|\varrho|> |D|^{k}$ if  $|\mathcal{F}_N| < |D|$.

As in the previous example, the latter language can be potentially applied to a statistical modeling of a stochastic process $\{X_t\}_{t\in {\mathbb Z}}, X_t\in [a,b]$. Suppose we have $N$ discretizations of $[a,b]$ with step sizes $h_1, \cdots, h_N$, such that $\frac{h_{i+1}}{h_{i}}=M_i\in {\mathbb N}$. 
Let us denote ${\rm round}_i(x) = \max\{th_{i}\mid t\in {\mathbb Z}, th_{i}\leq x\}$. Obviously, ${\rm round}_i(x)$ is a value of $x$ rounded w.r.t. $i$-th discretization.  
Then, we define $D = \{ih_1 \mid i\in {\mathbb Z}, i h_1\in [a,b]\}$ and $(ih_1, j h_1)\in \gamma^{N+1-s}\Leftrightarrow {\rm round}_s(ih_1)={\rm round}_s(jh_1), s\in [N]$. By construction, $\gamma^1\supset \cdots \supset \gamma^N$, and we define $\mathcal{F}_i = D/\gamma_i$.

After a sequence reversion in each relation of $\Gamma$, we obtain
\begin{equation*}
\Gamma' = \{ \Omega_1\times \cdots \times \Omega_l| \Omega_i \in \mathcal{F}_{n_i}, i\in [l], {n_1}\geq {n_2} \geq \cdots \geq {n_l}, l\leq r\}
\end{equation*}
In other words, $(X_{t_0}, \cdots, X_{t_0+l-1})\in \Omega_1\times \cdots \times \Omega_l\in \Gamma'$ is equivalent to stating that $X_{t_0+i-1}\in  \Omega_{i}, i\in [l]$ and the latter means that ${\rm round}_{n_{i}}(X_{t_0+i-1})$ is fixed to some specific value. Note that as the index $i\in [l]$ increases, we  round the value   $X_{t_0+i-1}$ more coarsly. Thus, the Algorithm~\ref{alg:DP2} computes a partition function and MAP for the following probabilistic model
\begin{equation*}
\begin{split}
P(x_1,\cdots,x_N)\propto \exp\big(\sum_{l=1}^r \sum_{N\geq n_1\geq \cdots\geq n_l\geq 1}\sum_{a_1\in F_{n_1}, \cdots,a_l\in F_{n_l}}
\sum_{i,j\in [n]: j-i+1=l} c^{n_1,a_1,\cdots, n_l,a_l}_{ij}[x_{i:j}\in a_1\times \cdots  \times a_l]\big).
\end{split}
\end{equation*}
in time ${\mathcal O} (n (|D|+1)^2 |\Pi|^2)$. In this generalized pattern-based model, the constraint $(X_{t_0},\cdots, X_{t_0+l-1})\in a_1 \times \cdots \times a_l$ is equivalent giving larger intervals of variation for $X_{t+i-1}$ as $i\in [l]$ increases. Again, as Figure~\ref{3trends} demonstrates, this type of ``local correlations'' in time-series modeling is natural, because any ``cause'' that precedes a certain event, has a predictable effect in the beginning, but leads to different divergent trajectories when approaching the end of pattern. In other words, interval of ``uncertainty'' increases with $i\in [l]$.
\end{example}

\begin{example} Suppose $\sim$ is an equivalence relation on $D^2$, $\mathcal{J} = D^2/\sim$ is a set of equivalence classes with respect to $\sim$, and $\Pi$ is any prefix-closed set of strings over $\mathcal{J}$. Thus, for any $\alpha\in \Pi$, $\alpha^-\in \Pi$. Let us denote
\begin{equation*}
\Gamma = \{ \varrho_{w, W} | w\in \Pi, |w|\ne 0, W\in 2^D\setminus \{\emptyset\}\} \cup (2^D\setminus \{\emptyset\})
\end{equation*}
where for $w=w_1\cdots w_{n-1}$, $w_i\in \mathcal{J}$ and $W\in 2^D\setminus \{\emptyset\}$ we denote:
\begin{equation*}
\varrho_{w_1\cdots w_{n-1}, W} = \{ (x_1, \cdots, x_n)\in D^n | (x_i,x_{i+1})\in w_i, i\in [n-1], x_n\in W\}.
\end{equation*}
It is easy to see that $\overline{\Gamma^{\cap}} = \Gamma$ and $|\Gamma| = |\Pi|(2^{|D|}-1)$. 

Using this language one can model patterns of ``growth'' and ``decline'' for a stochastic process $\{X_t\}_{t\in {\mathbb Z}}, X_t\in [d]$ with a state space $D = [d]$ in the following way (see Figure~\ref{waves}). For example, let $\mathcal{J} = \{\nearrow, \searrow, \rightarrow\}$, where $\nearrow = \{(a,b)\in D^2| a<b\}$, $\searrow = \{(a,b)\in D^2| a>b\}$ and $\rightarrow = \{(a,b)\in D^2| a=b\}$. It is easy to see that $D^2 = \nearrow\cup \searrow\cup \rightarrow$ and the union is disjoint.
Let $\Pi$ be a prefix-closed set of strings over $\mathcal{J}$. For example, $\Pi = \{\nearrow, \nearrow\nearrow,  \searrow, \searrow\searrow, \searrow\rightarrow\nearrow, \searrow\rightarrow\}$. We interpret each relation in $\Gamma$, e.g. $\varrho_{\scaleto{\searrow\rightarrow\nearrow}{3pt},D}$, as a ``growth-decline'' pattern of a process with a state-space $D$. 

The algorithm~\ref{alg:DP'} can be applied to the minimization of the following objective
\begin{equation*}
E(x_1, \cdots, x_N) = \sum_{i=1}^N w_{ia}[x_i=a]+\hspace{-20pt}\sum_{w\in \Pi, W\in 2^D\setminus \{\emptyset\}}\sum_{i,j\in [N]: j-i+1=|w|}c^{w,W}_{ij}[x_{i:j}\in \varrho_{w, W}],
\end{equation*}
where all weights $w_{ia}, c^{w,W}_{ij}$ are nonpositive. It has a complexity $\mathcal{O}(n|\Gamma|^2)$. Since this language is not simple, the algorithm~\ref{alg:DP2} has the complexity ${\mathcal O} (n |\Gamma^\ast|^2)$ and the cardinality of $\Gamma^\ast$ can be potentially very large.
\end{example}

\begin{figure}

  \begin{center}
\includegraphics[width=6cm]{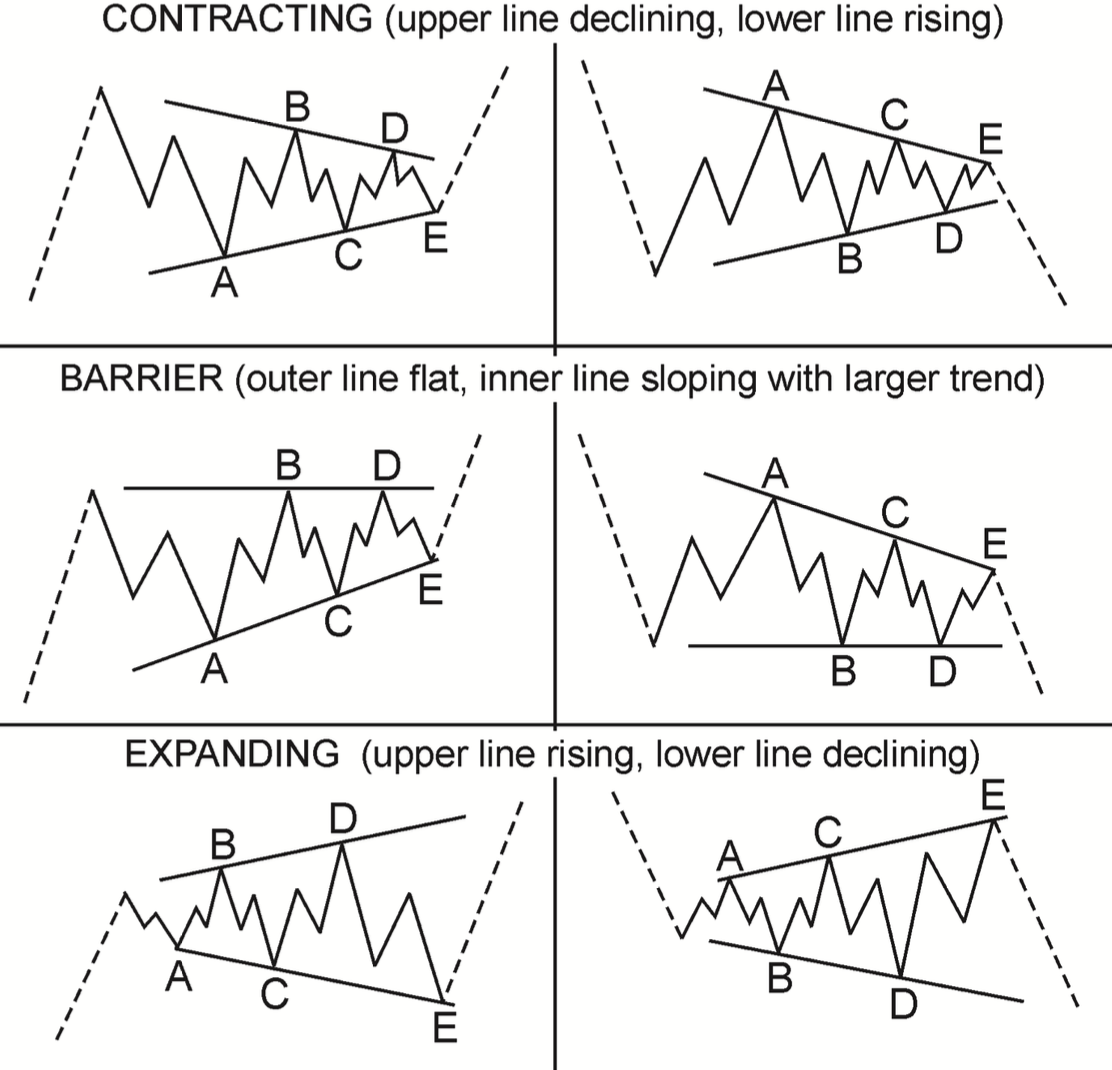}
  \end{center}
\caption{ A vocabulary of "growth-decline" patterns in financial markets, known as Elliot waves.}\label{waves}
\end{figure}

\section{Conclusions}
Probabilistic models based on a generalized pattern-based energy naturally appear in such fields as natural language processing, stochastic processes, conditional random fields etc. Key tasks that need to be solved in these applications include the minimization of an energy (MAP) and the computation of  a partition function. Both tasks can be formulated uniformly as the computation of a partition function over a semiring. Also,  the energy minimization task can be viewed as a Valued CSP whose dual constraint graph has a bounded treewidth and a language of constraints is fixed. 
 
A partition function and MAP of a generalized pattern-based energy can be computed exactly if a language of constraints $\Gamma$ is simple and has a moderate cardinality of the closure $\overline{\Gamma^\cap}$. The cardinality of $\overline{\Gamma^\cap}$ can be both moderate and very large for some natural languages. In this sense, the parameter $|\overline{\Gamma^\cap}|$ characterizes the time complexity of MAP/a partition function computation. From examples that we list in Section~\ref{Examples} it is clear that $|\overline{\Gamma^\cap}|$ can be substantially (exponentially) smaller than $|D|^{l_{\max}}$ and this makes our algorithms for the energy minimization faster than state-of-the-art algorithms for Valued CSPs with bounded treewidth.

\section*{Acknowledgments}
We would like to thank Anuar Sharafudinov for his help with writing a Java code that computes the closure of a simple language. %This research has been funded by Nazarbayev University under Faculty-development competitive research grants program for 2023-2025 Grant \#20122022FD4131, PI Zh. Assylbekov.

%Bibliography
\bibliographystyle{unsrt}  
\bibliography{lit}

\end{document}